%% file: arxiv.tex
\documentclass[12pt]{article}

\usepackage[vmargin=1.1in,hmargin=1.1in]{geometry}
\usepackage{setspace}
\usepackage{parskip}
\usepackage{amsmath}
\allowdisplaybreaks  
\usepackage{amssymb}
\usepackage{mathtools}

\usepackage[ruled,vlined,linesnumbered]{algorithm2e}

\usepackage{textcomp, gensymb}
\usepackage{gensymb}
\usepackage{enumitem}
\usepackage{graphicx}
\usepackage{scrextend}
\usepackage{blindtext}
\usepackage{caption}
\usepackage{url}
\usepackage{natbib}
\usepackage{subcaption}
\usepackage{circuitikz}
\usepackage{listings}
\usepackage{color}
 \usepackage{microtype}
\usepackage{graphicx}
\usepackage{multirow}
\usepackage{booktabs} 

\usepackage{amssymb}
\usepackage{pifont}
\usepackage{hyperref}

\newcommand{\specialcell}[2][c]{%
\begin{tabular}[#1]{@{}c@{}}#2\end{tabular}}
\usepackage{amsmath}

\usepackage{authblk}
\usepackage{amsmath}
\usepackage{amssymb}
\usepackage{mathtools}
\DeclarePairedDelimiter{\ceil}{\lceil}{\rceil}
\usepackage{amsthm}
\usepackage[capitalize,noabbrev]{cleveref}
\usepackage{comment}
\usepackage{dsfont}
\theoremstyle{plain}
\newtheorem{theorem}{Theorem}
\newtheorem{proposition}{Proposition}
\newtheorem{lemma}{Lemma}
\newtheorem{corollary}{Corollary}
\theoremstyle{definition}

\newtheorem{assumption}{Assumption}
\theoremstyle{remark}
\newtheorem{remark}{Remark}
\definecolor{codegreen}{rgb}{0,0.6,0}
\definecolor{codegray}{rgb}{0.5,0.5,0.5}
\definecolor{codepurple}{rgb}{0.58,0,0.82}
\definecolor{backcolour}{rgb}{0.95,0.95,0.92}

\newenvironment{keywords}
    {\vspace{1em}\noindent\textbf{Keywords: }}
    {\vspace{1em}}
\lstdefinestyle{mystyle}{
    backgroundcolor=\color{backcolour},   
    commentstyle=\color{codegreen},
    keywordstyle=\color{magenta},
    numberstyle=\tiny\color{codegray},
    stringstyle=\color{codepurple},
    basicstyle=\footnotesize,
    breakatwhitespace=false,         
    breaklines=true,                 
    captionpos=b,                    
    keepspaces=true,                 
    numbers=left,                    
    numbersep=5pt,                  
    showspaces=false,                
    showstringspaces=false,
    showtabs=false,                  
    tabsize=2
}
 
\DeclarePairedDelimiter\floor{\lfloor}{\rfloor}

\date{\today}
\usepackage{bbm}

\renewcommand{\P}{\mathbb{P}}
\newcommand{\N}{\mathbb{N}}
\usepackage{color}

\usepackage{cleveref}
\crefname{assumption}{assumption}{assumptions}
\Crefname{assumption}{Assumption}{Assumptions}

\crefname{theorem}{theorem}{theorems}
\Crefname{theorem}{Theorem}{Theorems}

\crefname{figure}{figure}{figures}
\Crefname{figure}{Figure}{figures}

\title{\vspace{-0.6cm}
\bfseries Distribution-free changepoint localization after sequential change detection}

\author[1]{Aytijhya Saha}
\author[2]{Aaditya Ramdas}

\affil[1]{Massachusetts Institute of Technology.  \texttt{aytijhya@mit.edu}}
\affil[2]{Carnegie Mellon University. 
\texttt{aramdas@cmu.edu}}
\begin{document}

\date{}
\maketitle

\input{abs}

\begin{keywords}
Changepoint localization, Post-detection inference, conformal p-values.
\end{keywords}

\input{main-content}

\appendix

\input{appendix-content}

\end{document}

%% file: abs.tex
\begin{abstract}
    This paper introduces a distribution-free framework for constructing post-detection confidence sets for changepoints after stopping a sequential change detection procedure. It is well known that conformal test martingales can be used to sequentially detect changes in distribution, but by themselves provide no inference for the time at which a proclaimed change occurred. Past work on post-detection inference requires pre- and post-change classes of distributions to be known, but this paper accomplishes localization of the changepoint without any distributional assumptions. We establish finite-sample coverage guarantees (conditional on correct detection). We provide non-asymptotic bounds on the conditional expected size of the confidence sets. Under suitable asymptotic regimes, we prove that the conditional expected size of the confidence set remains uniformly bounded and demonstrate strong empirical performance on simulated and real data. To the best of our knowledge, this is the first general distribution-free framework for sequential changepoint localization with valid post-detection coverage.
\end{abstract}

%% file: main-content.tex
\section{Introduction}
Consider the following general problem of sequential change analysis in an arbitrary space $\mathcal{X}$, which could be the Euclidean space $\mathbb R^d$, or images, or almost anything else. A sequence of $\mathcal{X}$-valued observations 
 ${\bf X} = X_1,X_2,\cdots$ are such that for some unknown changepoint $T\in\N\cup\{\infty\}$,
\begin{equation}\label{eq:data-setup}
X_1,X_2,\cdots,X_{T-1}\stackrel{}{\sim}F_0 ~\text{ and }~X_{T},X_{T+1},\cdots\stackrel{}{\sim}F_1,
\end{equation}
where $F_0$ and $F_1$ are any distinct pre- and post-change distributions. 

A large body of work has focused on designing stopping rules that rapidly detect distributional changes while controlling false alarms, leading to a rich literature on classical procedures such as CUSUM \citep{page1954continuous} and the Shiryaev--Roberts procedure \citep{shiryaev1963optimum,roberts1966comparison}. More recently, attention has shifted toward distribution-free sequential detection methods, including conformal test martingales \citep{vovk2021retrain}.

This paper studies the following fundamental question: after a sequential changepoint detector raises an alarm, can we construct statistically valid confidence sets for the unknown changepoint using only the data observed up to the stopping time, while making no distributional assumptions on the pre- and post-change distributions?

We follow the setup in \cite{saha2025post}; these authors studied this problem in a non-distribution-free setting, i.e.\ when the pre- and post-change sets of distributions can be prespecified and non-intersecting, but their abstract problem setup is still useful in our more general setting. Suppose the user has already decided to use a sequential changepoint detection algorithm $\mathcal A$, which is a mapping from the data sequence $\bf X$ to a stopping time $\tau = \mathcal A({\bf X})$ at which an alarm is triggered, indicating a potential change. When the procedure stops at time $\tau$, it is not known whether this corresponds to a false alarm, and even if a true change has occurred, the only available information is that it happened sometime before $\tau$, but not \emph{when} that change occurred. In this paper, we aim to construct a confidence set $\mathcal{C}$ for the changepoint $T$ based on data up to stopping time $\tau$ such that without any distributional assumption on the unknown $F_0$ and $F_1$,
\begin{equation}
\label{eq:cond-coverage-intro}
\P_{F_0,T,F_1} (T \in \mathcal C \mid \tau \geq T) \geq 1-\alpha,
\end{equation}
meaning that $\mathcal C$ only needs to cover if algorithm $\mathcal A$ stops only after the true changepoint $T$. 
 While such a conditional guarantee is natural, \cite{saha2025post} also showed that for detection algorithms with a bounded average run length to false alarm
(ARL2FA, or just ARL for short), no post-detection confidence set $\mathcal C$ can satisfy an unconditional coverage guarantee. However, for algorithms that control the probability
of false alarm (PFA) at level $\zeta$, our method can achieve an unconditional guarantee
\begin{equation}
\label{eq:uncond-coverage-intro}
\P_{F_0,T,F_1} (T \in \mathcal C ) \geq (1-\alpha)(1-\zeta).
\end{equation}
We primarily focus on conditional coverage throughout the paper.
Apart from the aforementioned work, the conditional coverage \eqref{eq:cond-coverage-intro} has been considered and established by several other prior works, for example, \cite{ding2003lower,WU20063625,wu2007inference}, but all of these works make restrictive assumptions on the pre- and post-change distributions --- most importantly that the pre- and post-change sets of distributions are known and do not intersect; we provide further discussion in \Cref{subsec:prior-1}.
In some practical deployments, however, such assumptions are often unrealistic.  In many real-world systems---such as industrial monitoring, network security, or streaming data applications---the data-generating mechanisms are complex, high-dimensional, and poorly specified, making reliable distributional modeling difficult or infeasible. Motivated by this gap, we develop a distribution-free framework for post-detection confidence sets for changepoints. Our approach enables valid uncertainty quantification for the changepoint, with the validity (coverage of the confidence set) not relying on any modeling assumptions, thereby making it more broadly applicable in real-world settings. Like conformal prediction, we are allowed to use models, but they only affect the statistical efficiency and not the validity of the procedure. Our inference method can therefore be used as a wrapper around any type of detection algorithm, including procedures based on ARL control, PFA control, and sequential testing.

Our approach builds on conformal test martingales \citep{vovk2003testing,vovk2021retrain} to construct distribution-free confidence sets. Specifically, we assume exchangeability within the pre- and post-change segments, but do not require knowledge of the underlying distributions. While Vovk's original method was only for \emph{detecting} changes, we nontrivially extend those ideas to work as a wrapper around a detection procedure in order to \emph{localize} the changepoint after detection.

Our construction proceeds by separately constructing lower and upper confidence sets and considering their intersection as a two-sided set. This decomposition is not merely technical; test martingales for one are powerless for the other, and one must do something qualitatively slightly different to obtain nontrivial lower and upper confidence sets. 

We note that lower and upper confidence sets by themselves have direct practical interpretations and applications. Lower confidence sets have been previously studied in \cite{ding2003lower,wu2007inference} in restricted parametric settings.
A lower confidence set yields a lower confidence bound $L$ for $T$ so that all observations before $L$ can be treated as pre-change with controlled risk, which is particularly useful in quality control and process monitoring. In such settings, a practitioner is often less concerned with pinpointing the exact change time and more concerned with isolating ``the pre-change part'' with high confidence, i.e., estimating a safe time threshold before which the system can be treated as pre-change. For instance, in industrial manufacturing or sensor-based monitoring systems, the lower confidence bound can be interpreted as a conservative estimate of the earliest possible failure or degradation onset. This allows operators to take corrective measures (e.g., isolating and inspecting the potentially affected products) with controlled statistical risk. On the other hand, an upper confidence set provides a guaranteed subset of observations that are post-change with high confidence. All observations following the largest element of the upper confidence set can be treated as post-change data with high confidence and used to retrain, recalibrate, or adapt machine-learning systems to the new operating regime.
In contrast, the two-sided confidence set is more appropriate when the goal is to localize the changepoint itself. For instance, rapid shifts in public sentiment often occur on online platforms: a product failure may trigger a sudden surge of negative reviews, or a major event may abruptly alter the tone of social media discussions. In such cases, identifying when the change occurred is crucial for investigating emerging issues and the reason for the change.

\subsection{Prior work on post-detection confidence sets for $T$} 
\label{subsec:prior-1}
We are aware of three directly related works: \cite{ding2003lower}, \citet[Chapter 3]{wu2007inference}, and \cite{saha2025post}.

 Both  \cite{ding2003lower,wu2007inference} provide asymptotically valid inference after the CUSUM procedure \citep{page1954continuous} declares a change for known $F_0$ and $F_1$ in an exponential family; the former only constructs a lower confidence bound for $T$ while the latter provides a (lower and upper) confidence set for $T$.
Their approaches have notable methodological and theoretical limitations: (i) they are tailored specifically to the CUSUM procedure, (ii) they assume that the pre- and post-change distributions are both known exactly and that they belong to an exponential family,  (iii) they only provide asymptotic coverage guarantees (valid when both $T\to\infty$ and ARL $\to\infty$). 

Recently, \cite{saha2025post} proposed nonasymptotically valid confidence sets for $T$ that are more general than those in the aforementioned past works: in particular, without making any assumptions on the detection procedure. Their confidence sets can handle changes from an unknown $F_0 \in \mathcal{P}_0$ to some unknown $F_1 \in \mathcal{P}_1$ on a general space (e.g., $\mathbb R^d$), but their drawbacks are the following: (i) while no assumptions are needed when $\mathcal P_0$ is a singleton, they make some restrictive assumptions on $\mathcal P_0$ in the composite setting, (ii) their approach is not distribution-free in that they need $\mathcal{P}_0 \cap \mathcal{P}_1 = \emptyset$, and $\mathcal{P}_0 \cup \mathcal{P}_1\neq \mathcal P$ (where $\mathcal P$ is the universe of all distributions on that space), meaning that we must both restrict and partition the set of all distributions. In contrast, our work will have $\mathcal{P}_0 = \mathcal{P}_1 = \mathcal P$, with only the obvious restriction that $F_0 \neq F_1$.

Our framework circumvents the aforementioned drawbacks and is the \textit{first distribution-free, theoretically sound, and widely applicable practical tool} for sequential change localization. 

\subsection{Conformal inference and test of exchangeability}
Our approach builds on the construction of conformal test martingales under the null hypothesis of exchangeability~\citep{vovk2003testing}. Prior works, such as \cite{saha2024testing,ramdas2022testing}, have proposed other tests of exchangeability; however, these focus on alternative hypotheses that differ from the changepoint setting considered here, making them less suitable for our purposes. Our techniques are most closely related to \cite{vovk2003testing,vovk2021testing,vovk2021retrain}. To the best of our knowledge, we are the first to leverage those ideas for post-detection changepoint \emph{localization} (as opposed to sequential \emph{testing} of exchangeability, or sequential changepoint \emph{detection}).

In a separate direction, \cite{dandapanthula2025offline} introduced distribution-free \textit{offline} methods for changepoint localization (i.e.,\ confidence sets for the changepoint) based on conformal p-values, and \cite{bhattacharyya2025theoretical} analyzed the size of the resulting confidence sets. Subsequently, \cite{hore2026conformal} proposed tighter confidence sets within the same distribution-free \textit{offline} framework. 

We emphasize that the above offline inferential approaches are not valid in the sequential regime considered in this work because sequential detection procedures are usually stopping times, making the number of observations data-dependent (while offline procedures assume a fixed, deterministic number of samples $n$), and if one tries to condition on the stopping time, the distributions of the observed data becomes intractable, and we immediately lose (separate) exchangeability of the pre-change and post-change observations, making offline techniques inapplicable.

\subsection{Our Contribution}
In this work, we develop a general distribution-free framework for post-detection changepoint localization in sequential settings. Sequential change detection and changepoint inference are fundamentally different problems. Detection asks whether a change has occurred; localization asks when it occurred. We show these can be decoupled: arbitrary sequential detectors can be combined with distribution-free inferential wrappers to yield valid changepoint uncertainty quantification. Our contributions are summarized below.
\begin{itemize}
    \item \textbf{Distribution-free post-detection inference.}
Our framework for constructing post-detection confidence sets for changepoints with finite-sample validity guarantees does not rely on assumptions (beyond exchangeability) on the pre- and post-change models. In contrast to prior work, our framework does not require specifying disjoint pre- and post-change model classes.
\item \textbf{One-sided and two-sided confidence sets.}
We develop separate constructions for lower and upper post-detection confidence sets and show how to combine them into valid two-sided confidence sets for the changepoint. The lower and upper constructions require fundamentally different sequential testing strategies, reflecting the intrinsic asymmetry of sequential localization problems. We prove finite-sample coverage guarantees for all the proposed confidence sets conditional on the detector stopping after the true changepoint in \Cref{thm:cov-lower,thm:cov-upper,cor:cov-twoside}.
\item \textbf{Non-asymptotic bound and asymptotic sharpness guarantees on size.}
We derive explicit bounds on the size of the resulting confidence sets and show that, under suitable signal conditions and detection-delay regimes, the expected confidence set width remains uniformly bounded asymptotically. We also provide finite-sample (non-asymptotic) guarantees. These results demonstrate that reliable changepoint localization is possible even without distributional assumptions.
\end{itemize}

The rest of the paper is organized as follows. In \Cref{sec:method}, we present our main methodology, where we first construct lower and upper confidence sets in \Cref{sec:lower} and \Cref{sec:upper}, respectively, and finally propose a two-sided confidence set in \Cref{sec:twosided}. We provide finite-sample and asymptotic bounds on the conditional expected width of the confidence sets with an oracle calibrator in \Cref{sec:theory}.  We then extend these theoretical results to the practical setting with a learned calibrator in \Cref{sec:theory-real}.  We validate our method on simulated and real data in \Cref{sec:expt}. We conclude the paper in \Cref{sec:conc}. Proofs of all theoretical results presented in the main text, along with additional theoretical and experimental results, are provided in the Appendix Materials.

\section{Method}
\label{sec:method}

\subsection{Problem setup}

For $t \in \N$, let $\P_{F_0,t,F_1}$ denote joint distribution of the sequence $\{X_n\}_{n}$ when  $X_1,\dots,X_{t-1}\stackrel{}{\sim}F_0$ and $X_t,X_{t+1},\dots\stackrel{}{\sim}F_1$. 
As a corner case, let $\P_{F_0,\infty}$ be the joint distribution when $X_1,X_2,\dots \stackrel{}{\sim}F_0$. 

We now state the assumptions underlying our results. The first assumption is required for constructing the lower confidence set, while the second is needed for the upper confidence set. Both assumptions are jointly required for the validity of the two-sided confidence set.
\begin{assumption}
\label{assmp-post}
The post-change sequence $X_T,X_{T+1},\cdots$ is exchangeable. 
\end{assumption}
\begin{assumption}
\label{assmp-pre}
The pre-change data $X_1,\cdots,X_{T-1}$ is exchangeable.
\end{assumption}

This is a standard assumption for any distribution-free method to work, see e.g., \cite{dandapanthula2025offline,hore2026conformal}. However, these prior works additionally require independence between the pre-change segment $X_1,\cdots,X_{T-1}$ and the post-change part $X_T,X_{T+1},\cdots$, an assumption that is not needed in our framework.

We first construct upper and lower confidence sets and then consider their intersection to derive a two-sided confidence set. The overall approach is summarized below. For both the upper and lower confidence sets, the underlying idea is the same: for each $t\in\{1,\cdots,\tau\}$ we test the same null hypothesis
\[H_{0,t}: T=t \text{ , i.e., } t \text{ is the true changepoint},\]
but against different alternatives $T> t$ for the lower confidence set and $T< t$ for the upper confidence set. In classical statistics, inverting level-$\alpha$ tests yields $1-\alpha$ confidence sets. However, in this case, we need to calibrate the level properly since the number of hypotheses, $\tau$, is data-dependent, and we are interested in obtaining a conditional coverage guarantee, conditioned on $\tau\geq T$. For this calibration, we will need to assume black box access to the change detection algorithm (by ``black-box”, we mean that we can run the algorithm on any input data
sequences) as in \cite{saha2025post}.


\subsection{Lower confidence set}
\label{sec:lower}
The first goal is to construct a valid lower confidence set for the changepoint $T$ using observations up to the stopping time $\tau$. For any candidate $t \le \tau$, we test whether $t$ could plausibly be the true changepoint ($T=t$) against the alternative that the true changepoint is after $t$, i.e., $T>t$. 

If $T>t$, then the segment $X_t, \ldots, X_\tau$ contains both pre- and post-change observations, violating exchangeability and thus providing evidence against $t$. Conversely, under the null hypothesis $T=t$,  $X_j:t\leq j\leq n$ lies entirely in the post-change regime and remains exchangeable, for any fixed $n\geq t$, but not for data-dependent $n$ like our stopping time $\tau$. Therefore, permutation-based ideas in offline changepoint localization (e.g., \cite{hore2026conformal}) are not valid in our setting. Therefore, we take the approach of constructing conformal test martingales \cite{vovk2003testing,vovk2005algorithmic}. As is common in the conformal inference literature, we consider some univariate score function $S$ and transform the raw data into its scores $\{S(X_i)\}_{1\leq i\leq \tau}$ (see discussion in \Cref{sec:score} on how to choose score functions).
Starting at each fixed $t$, we compute sequential conformal p-values based on score comparisons:
\begin{align}
\label{eq:forward-p-val}
   \nonumber p^j_t=&\frac{1}{j-t+1}\Bigg[\sum_{t\leq i\leq j}\mathds{1}(S(X_j)>S(X_i))\\
    &+U_{t,j}\sum_{t\leq i\leq j}\mathds{1}(S(X_j)=S(X_i))\Bigg],
\end{align}
for $j\geq t$ and $U_{t,j}\stackrel{i.i.d.}{\sim}\text{Uniform}(0,1)$. While the method is valid for any choice of score function $S$, efficiency, as measured by the set size of the confidence set, would depend on that. We discuss it further in \Cref{sec:score}.
The sequence of p-values $\{p^j_t\}_{j\geq t}$ is known to be i.i.d.\ Uniform(0,1) \citep{vovk2003testing,vovk2021retrain}.
Hence, we can transform p-values into e-values using some non-negative predictable function $f_{t,j}$ such that $\int_{0}^1f_{t,j}(p)dp\leq 1$ (which we call a calibrator) and aggregate. One particular choice of calibrator that we have used in our experiments and is common in the literature (e.g., \cite{shaer2026testing}) is the following:
\[f_{t,j}(p)=1+\lambda_{t,j}(p-0.5),\]
where the hyperparameter $\lambda_{t,j}\in[-1,1]$ can be chosen adaptively using online methods. One possible approach for learning the sequence ${\lambda_{t,j}}_{j\ge t}$ is given by the online Newton step \citep{hazan2006logarithmic} with the follow-the-leading-history \citep{hazan2007adaptive} procedure described in \Cref{alg:flh-ons-calibrator}.
Candidates $t$ for which the accumulated evidence against them is not too large are retained in the lower confidence set:
\begin{equation}
\label{ci-lower}
    \mathcal{C}_{low}^\alpha=\left\{2\leq t\leq \tau:\max_{i:t\leq i\leq\tau}\prod_{j=t}^i f_{t,j}(p^j_t) \leq\frac{1}{\alpha\hat r_t}\right\},
\end{equation}
where $\hat r_t$ is some distribution-free unbiased (or negatively biased) estimate $\P_{F_0,\infty}[\tau\geq t]$, which is independent of the p-values $\{p^j_t\}_{j\geq t}$. We discuss constructions of such estimators in Section \ref{sec:rt}. The algorithm is summarized in Algorithm \ref{alg:lower}.
Next, we shall prove that the lower confidence set achieves at least $1-\alpha$ conditional coverage.

\begin{algorithm}[ht]
\caption{Lower Confidence Set for Changepoint}
\label{alg:lower}
\KwIn{Data $X_1, \dots, X_\tau$ ($\tau$ is a stopping time), level $\alpha$, score function $S$, data-independent unbiased or negatively biased estimator $\hat r_t$ of $\P_{F_0,\infty}[\tau\geq t]$, p-to-e calibrator $\{f_{t,j}\}_{t\leq j\leq \tau,1\leq t\leq \tau}$.}
\KwOut{Lower confidence set $\mathcal{C}_{low}^\alpha$.}
$\mathcal{C}_{low}^\alpha\gets\emptyset$\;

\For{$t = 2$ \KwTo $\tau$}{
    $M_t\gets 1, ~j\gets t$\;
    
    \While{$M_t\leq\frac{1}{\alpha\hat r_t}$ \textbf{and} $j\leq\tau$}{
        Draw $U_{t,j}\sim$ Uniform(0,1)\;
        Compute $p^j_t$ as defined in \eqref{eq:forward-p-val}\;
        $M_t\gets M_t\times f_{t,j}(p_t^j),~ j\gets j+1$\;
    }
    \If{$M_t\leq\frac{1}{\alpha\hat r_t}$}{
        Add $t$ to set $\mathcal{C}_{low}^\alpha$\;
    }
}
\Return{$\mathcal{C}_{low}^\alpha$}
\end{algorithm}

\begin{theorem}
\label{thm:cov-lower}
   Under \Cref{assmp-post}, for {any} $\alpha\in(0,1), T\in \N\setminus\{1\}$, pre- and post-change distributions $F_0$ and $F_1$ respectively, and change detection method $\mathcal{A}$ with stopping time $\tau$ satisfying $\mathbb \P_{F_0,\infty}(\tau\geq T)\neq 0$, the confidence set $\mathcal{C}_{low}^\alpha$ defined in \eqref{ci-lower}, constructed using any data-independent unbiased or negatively biased estimator $\hat r_t$ of $\mathbb P_{F_0,\infty}[\tau \geq t]$, satisfies $\P_{F_0,T,F_1}(T\in \mathcal{C}_{low}^\alpha\mid \tau\geq T)\geq 1-\alpha.$
\end{theorem}

Note that if the detector controls the probability of false alarm (PFA) at level $\zeta$, i.e., 
        \[
        \P_{F_0,\infty}(\tau<\infty)\leq \zeta,
        \]
        then the conditional guarantee would imply the unconditional coverage as well:
\begin{align}
\label{eq:uncond-cov}
  \nonumber  &\P_{F_0,T,F_1}(T\in \mathcal{C}_{low}^\alpha)\\
 \nonumber    &\geq \P_{F_0,T,F_1}(T\in \mathcal{C}_{low}^\alpha\mid \tau\geq T)\mathbb P_{F_0,\infty}[\tau \geq T] \\
 \nonumber    &\geq (1-\alpha)\mathbb P_{F_0,\infty}[\tau = \infty]\\
    &\geq (1-\alpha)(1-\zeta).
\end{align}

Next, we analogously discuss the upper confidence set.

\subsection{Upper confidence set}
\label{sec:upper}
One can similarly construct an upper confidence set for the changepoint $T$. For any candidate $t \le \tau$, we test whether  $T=t$ against the alternative that the true changepoint is before $t$, i.e., $T<t$ (as opposed to $T>t$ considered in the previous section). So,
starting at each fixed $t$, we compute sequential conformal p-values in the \emph{backward} direction based on score comparisons:
  \begin{align}
  \label{eq:backward-p-val}
     \nonumber q^j_t&=\frac{1}{t-j+1}\bigg[\sum_{j\leq i\leq t}\mathds{1}(S(X_j)>S(X_i))\\
      &+U_{t,j}^\prime\sum_{j\leq i\leq t}\mathds{1}(S(X_j)=S(X_i))\bigg],~ 
  \end{align}
    for $j\leq t$ and $U_{t,j}^\prime\stackrel{i.i.d.}{\sim}\text{Uniform}(0,1)$. 
Candidates $t$ for which the accumulated evidence against them is not too large are retained in the upper confidence set:
  \begin{equation}
\label{ci-upper}
\mathcal{C}_{up}^{\beta}=\left\{2\leq t\leq \tau:\max_{i:1\leq i\leq t-1}\prod_{j=i}^{t-1} g_{t-1,j}(q^j_{t-1}) \leq\frac{1}{\beta\hat r_t}\right\},
\end{equation}
where $\{g_{t,j}\}_{j\leq t}$ are calibrators (i.e., non-negative predictable functions such that $\int_{0}^1g_{t,j}(p)dp\leq 1$) and $\hat r_t$ is some distribution-free unbiased (or negatively biased) estimate $\P_{F_0,\infty}[\tau\geq t]$, which is independent of the p-values. We discuss constructions of such estimators in Section \ref{sec:rt}. The algorithm is summarized in Algorithm \ref{alg:upper}.
Next, we shall prove that the upper confidence set achieves at least $1-\beta$ conditional coverage.

\begin{algorithm}[!h]
\caption{Upper Confidence Set for Changepoint}
\label{alg:upper}
\KwIn{Data $X_1, \dots, X_\tau$, level $\beta$, score function $S$, data-independent unbiased or negatively biased estimator $\hat r_t$ of $\P_{F_0,\infty}[\tau\geq t]$, p-to-e calibrator $\{g_{t,j}\}_{1\leq j\leq t,1\leq t\leq \tau}$.}
\KwOut{Upper confidence set $\mathcal{C}_{up}^{\beta}$}
$\mathcal{C}_{up}^{\beta}\gets \emptyset$\;

\For{$t = 2$ \KwTo $\tau$}{
    $M_t\gets 1, ~j\gets t-1$\;
    
    \While{$M_t\leq\frac{1}{\beta\hat r_t}$ \textbf{and} $j\geq 1$}{
        Draw $U_{t-1,j}^\prime\sim$ Uniform(0,1)\;
        Compute $q^j_{t-1}$ as defined in \eqref{eq:backward-p-val}\;
        
        $M_t\gets M_t\times g_{t-1,j}(q_{t-1}^j),~ j\gets j-1$\;
    }
    \If{$M_t\leq\frac{1}{\beta\hat r_t}$}{
        Add $t$ to set $\mathcal{C}_{up}^{\beta}$\;
    }
}
\Return $\mathcal{C}_{up}^{\beta}$\;
\end{algorithm}

\begin{theorem}
\label{thm:cov-upper}
    Under \Cref{assmp-pre}, for {any} $\beta\in(0,1), T\in \N\setminus\{1\}$, pre- and post-change distributions $F_0$ and $F_1$ respectively, and change detection method $\mathcal{A}$ with stopping time $\tau$ satisfying $\mathbb \P_{F_0,\infty}(\tau\geq T)\neq 0$, the confidence set $\mathcal{C}_{up}^{\beta}$ defined above, constructed using any data-independent unbiased or negatively biased estimator $\hat r_t$ of $\mathbb P_{F_0,\infty}[\tau \geq t]$, satisfies $\P_{F_0,T,F_1}(T\in \mathcal{C}_{up}^{\beta}\mid \tau\geq T)\geq 1-\beta.$
\end{theorem}

\subsection{Two-sided confidence set}
\label{sec:twosided}
We now combine the one-sided constructions to obtain a two-sided confidence set for the changepoint. Specifically, define
\begin{equation}
\label{ci-bothside} \mathcal{C}=\mathcal{C}_{low}^{\alpha}\cap\mathcal{C}_{up}^{\beta}.
\end{equation}
As a direct consequence of \Cref{thm:cov-lower,thm:cov-upper}, we have the following result:
\begin{corollary}
\label{cor:cov-twoside}
 Under \Cref{assmp-pre,assmp-post}, for {any} $\alpha,\beta\in(0,1), T\in \N\setminus\{1\}$, pre- and post-change distributions $F_0$ and $F_1$ respectively, and change detection method $\mathcal{A}$ with stopping time $\tau$ satisfying $\mathbb \P_{F_0,\infty}(\tau\geq T)\neq 0$, the confidence set $\mathcal{C}$ defined in \eqref{ci-bothside} satisfies $\P_{F_0,T,F_1}(T\in \mathcal{C}\mid \tau\geq T)\geq 1-\alpha-\beta.$ 
\end{corollary}

\begin{remark}
    Analogous to \eqref{eq:uncond-cov}, if the  detector controls PFA at level $\zeta$, the conditional coverage in \Cref{cor:cov-twoside} would imply the following unconditional guarantee:
    \[\P_{F_0,T,F_1}(T\in \mathcal{C})\geq (1-\alpha-\beta)(1-\zeta).\]
\end{remark}

In particular, one can choose $\beta=\alpha$ and then $\mathcal{C}$ defined in \eqref{ci-bothside} takes the form
\begin{equation}
\label{ci-bothside-2alpha}
\begin{aligned}
\mathcal{C}
=\Biggl\{&2\le t\le \tau :
\max\Biggl\{
\max_{i:t\le i\le\tau}\prod_{j=t}^i f_{t,j}(p_t^j),\\
&\max_{i:1\le i\le t-1}\prod_{j=i}^{t-1} g_{t-1,j}(q_{t-1}^j)
\Biggr\}
\le \frac{1}{\alpha\hat r_t}
\Biggr\},
\end{aligned}
\end{equation}
which gives $1-2\alpha$ conditional coverage.
We study the conditional length of the confidence sets in \Cref{sec:theory,sec:theory-real}.

\subsection{Constructing $\hat{r}_t$}
\label{sec:rt}
We now discuss how one can obtain unbiased estimates $\hat r_t$ of the quantity $\P_{F_0,\infty}[\tau\geq t]$ for each $t\leq\tau$ without making any distributional assumptions about the unknown $F_0$.

First, suppose the detector controls the probability of false alarm (PFA) at level $\zeta<1$, that is, \[ \P_{F_0,\infty}(\tau<\infty)\leq \zeta . \]
Then, $\P_{F_0,\infty}[\tau\geq t]\geq \P_{F_0,\infty}[\tau=\infty]\geq 1-\zeta$. Therefore, in this case, we may use the deterministic conservative choice 
\[ \hat r_t = 1-\zeta . \]
Otherwise, when the change detector controls ARL, instead of PFA (i.e., when PFA is 1), we will need to assume black-box access to the change detection algorithm.
In a distribution-free setting, the following assumption is natural.
\begin{assumption}
\label{assmp-2}
 The change detection algorithm is such that the distribution of the stopping time $\tau$ under no change does not depend on $F_0$.
\end{assumption}

We show that any change detection procedure based on online conformal p-values (e.g., the CUSUM or Shiryaev–Roberts–type detectors in \cite{vovk2021retrain}) satisfies this assumption.

\begin{proposition}
\label{prop-rt}
Any change detection procedure based on online conformal p-values (i.e., any stopping time that is a measurable function of online conformal p-values) satisfies \Cref{assmp-2}.
\end{proposition}

An immediate implication is that, when the assumption holds, the quantity $\hat r_t$ can just be estimated via Monte Carlo simulation using i.i.d. data streams from any convenient distribution (e.g., i.i.d. Uniform$(0,1)$ sequences). The following proposition follows immediately.

\begin{proposition}
For any distribution $F$, $B\in\N$, draw B many sequences $\{Y^{(i)}_n\}_n$ for $i=1,\cdots,B$. Let $\tau^{(i)}$ denote the stopping time for $\{Y^{(i)}_n\}_n$. Then, under \Cref{assmp-2},
 $\hat{r}_t=\frac{1}{B}\sum_{i=1}^B \mathds{1}(\tau^{(i)} \geq t)$ is an unbiased or negatively biased estimator of $\P_{F_0,\infty}[\tau\geq t]$.
\end{proposition}

In many applications, the change detection procedure may have finite ARL (i.e., PFA equal to one) and need not be distribution-free, while one may still seek distribution-free post-detection guarantees. The aforementioned approaches do not apply in this setting. We address this setting under an alternative assumption in \Cref{sec:app-rt-traing-data} in the Appendix.

\subsection{Choice of score functions}
\label{sec:score}
Our methods are valid for any choice of score function. However, achieving good power critically depends on selecting an appropriate score. As shown in \cite{dandapanthula2025offline}, the likelihood ratio is an optimal choice. Notably, our framework depends only on the relative ordering of the scores, rather than their absolute values, so any monotone transformation of the likelihood ratio is sufficient. 

In practice, however, even such transformations may be unknown. Our framework allows the score function to vary with both the candidate changepoint $t$ and the time index $j$, although for simplicity we used a common score function $S$ throughout the previous sections. For example, when constructing lower confidence sets, one may estimate the pre-change distribution using the observations $X_1,\ldots,X_{t-1}$ and update an estimate of the post-change distribution online using $X_t,\ldots,X_j$ for each $j\ge t$. An analogous construction can be used for upper confidence sets.
To ensure the finite-sample validity of our confidence sets, however, the score function must be constructed in a way that preserves the exchangeability assumptions underlying the conformal p-values. More precisely, for each $t$ and $j$, conditional on all randomness and data used to construct score $S_{t,j}$, the observations whose conformal ranks are computed must remain exchangeable under the corresponding null hypothesis.

However, it is quite common to use a pre-trained score function based on pre-training data that can be used to approximate the likelihood ratio. In such settings, predetermined/pre-trained score functions can be both effective and computationally convenient.

\section{Theoretical guarantees on the oracle confidence set}
\label{sec:theory}

\subsection{Bound on the pre-change length for lower confidence set}
We first state the assumptions required for our analysis.
\begin{assumption}
\label{assmp:indep}
    All the observations $X_i$'s are independent.
\end{assumption}
\begin{assumption}
\label{assmp:lower-score}
 The score $S(\cdot)$ is such that scores are all distinct almost surely and for independent draws $Y \sim F_0$ and $Z \sim F_1$, there exists a constant $\delta > 0$ such that: $|\mathbb{P}(S(Z) > S(Y)) -\frac{1}{2} |= \delta.$
\end{assumption}

This is a mild assumption requiring that the score function possesses nontrivial discriminative power between the pre- and post-change distributions.


\begin{assumption}
\label{assmp:lower-calib}
    For each $t$ and $j\geq t$, the calibrator used is $f_{t,j}(p)=1 + \lambda_{t,j}^*(p - 0.5)$ where $\lambda_{t,j}^* \in \arg\max_{\lambda \in [-1, 1]} \mathbb{E}\left[\log(1 + \lambda(p^{j}_{t} - 0.5))\right]$.
\end{assumption}

This assumption simplifies the analysis by using the oracle-optimal $\lambda^*$ rather than an empirical, sequentially learned $\lambda$. We remove this assumption in \Cref{sec:theory-real}, where the parameters $\lambda_{t,j}$ are chosen adaptively using online algorithms with logarithmic adaptive regret, such as the Follow-The-Leading-History (FLH) with
Online Newton Step (ONS) procedure in \Cref{alg:flh-ons-calibrator}.

In order to understand the next result, the following context is useful. In standard changepoint detection~\citep{tartakovsky2014sequential}, one often considers the asymptotic behavior of the detection delay, where the asymptotics consider a sequence of change detectors whose average run length increases to infinity. For the detector to actually produce a correct detection before a false detection, its average run length (ARL) would have to be $\Omega(T)$ (that is, if the ARL was $o(T)$, it would likely produce a false alarm before seeing the true change at $T$). In such settings, the optimal detection delay of any detector scales logarithmically with the average run length, i.e., $\Omega(\log T)$. The following result is to be read with the preceding context in mind.

\begin{theorem}
\label{thm:lower-asymp}
 Fix any $s\in(0,\infty)$ and let $C_0=160 + 128(1+\log 2)^2$. Consider a sequence of problems where the changepoint $T\uparrow \infty$, and the change detector also changes with $T$ such that
$$L:=\liminf_{T\to\infty}\mathbb{P}_{F_0,T,F_1}(\tau \ge T+s\log T)>0.$$ Then, under \Cref{assmp:indep,assmp:lower-score,assmp:lower-calib}, for the lower confidence set $\mathcal{C}_{\mathrm{low}}^\alpha$ defined in \eqref{ci-lower}, with $\hat r_t$ replaced by its oracle counterpart $r_t:=\P_{F_0,\infty}[\tau\geq t]$, we have
\begin{align*}
   & \mathbb{E}_{F_0,T,F_1}(|\mathcal{C}_{low}^\alpha\cap[T-1]|\mid \tau \ge T+s\log T) \\
   &=\begin{cases}
    O(1), & s\delta^4>C_0\\
    O(\log T), & s\delta^4=C_0\\
    O(T^{1- s\delta^{4}/C_0}), &  s\delta^4<C_0.
\end{cases}
\end{align*}
Consequently, in all these three cases, $\mathbb{E}_{F_0,T,F_1}\left(\frac{|\mathcal{C}_{low}^\alpha\cap[T-1]|}{T-1}\mid \tau \ge T+s\log T\right) =o(1)$.
\end{theorem}

\begin{remark}
The expected ``relative'' set size is conditionally guaranteed to be $o(1)$ regardless of the choice of score function, provided that the score is informative, i.e., $\delta>0$.
For a good score function, one has
$\delta\asymp 
\mathrm{TV}(F_0,F_1),$
where $\mathrm{TV}(F_0,F_1)$ denotes the total variation distance between $F_0$ and $F_1$ (See \Cref{thm:score-tv} in the Appendix Materials for a detailed result). Moreover, when $F_0$ and $F_1$ are supported on a common finite alphabet and all probabilities are uniformly bounded away from zero, Pinsker's inequality together with reverse Pinsker inequalities implies that $\mathrm{KL}(F_0\|F_1)
\asymp
\mathrm{TV}(F_0,F_1)^2,$
and hence $\mathrm{KL}(F_0\|F_1)\asymp\delta^2.$
On the other hand, classical lower bounds for sequential change detection imply that any procedure with average run length $O(T)$ must incur a detection delay of at least order $\frac{\log T}{\mathrm{KL}(F_0\|F_1)}.$
This suggests choosing $s = \frac{C}{\delta^2},$
for some constant $C>0$.
Under this choice, the resulting rate for the absolute set size exhibits different regimes depending on whether
$\delta^2$ is smaller than, larger than, or equal to $C_0/C$, while the ``relative'' set size is always $o(1)$.
\end{remark}

\begin{remark}
The conditioning event ${\tau \ge T+s\log T}$ is imposed to ensure that the detector observes a nontrivial number of post-change samples before stopping. Conditioning only on ${\tau \ge T}$ merely guarantees that no alarm is raised before the changepoint, but does not preclude the possibility that $\tau-T$ is very small. In such cases, there may be insufficient post-change information to reliably distinguish pre- and post-change observations, making it impossible to obtain nontrivial control on the pre-change length of the lower confidence set. By contrast, the event ${\tau \ge T+s\log T}$ guarantees the availability of at least $O(\log T)$ post-change observations, enabling exponential discrimination between pre- and post-change samples and leads to the sharpness bound.
\end{remark}

\begin{remark}
The assumption $\liminf_{T\to\infty}\mathbb{P}_{F_0,T,F_1}(\tau \ge T+s\log T)>0$ requires that the detector stops after at least $s\log T$ observations after the changepoint with nonvanishing probability. It holds when the detection delay is at least logarithmic in $T$ with positive asymptotic probability. It arises in a doubly asymptotic regime where the detection threshold grows appropriately with $T$.
\end{remark}

\begin{remark}
The assumption
$\liminf_{T\to\infty}\mathbb{P}_{F_0,T,F_1}(\tau \ge T+s\log T)>0$
is implied by the two conditions
\[
\liminf_{T\to\infty}\mathbb{P}_{F_0,T,F_1}(\tau \ge T+s\log T\mid\tau\ge T)>0
\]
and
\[
\liminf_{T\to\infty}\mathbb{P}_{F_0,T,F_1}(\tau\ge T)>0.
\]
Indeed,
$\mathbb{P}_{F_0,T,F_1}(\tau \ge T+s\log T)
=
\mathbb{P}_{F_0,T,F_1}(\tau \ge T+s\log T\mid\tau\ge T)
\,
\mathbb{P}_{F_0,T,F_1}(\tau\ge T),$
since $\{\tau\ge T+s\log T\}\subseteq\{\tau\ge T\}$. Therefore, if both factors are bounded away from zero asymptotically, then so is
$\mathbb{P}_{F_0,T,F_1}(\tau \ge T+s\log T)$.
If, in addition,
$\mathbb{P}_{F_0,T,F_1}(\tau\ge T)\to1,$
then
\[
\liminf_{T\to\infty}\mathbb{P}_{F_0,T,F_1}(\tau \ge T+s\log T)>0
\]
is equivalent to
\[
\liminf_{T\to\infty}
\mathbb{P}_{F_0,T,F_1}(\tau \ge T+s\log T\mid\tau\ge T)>0.
\]
\end{remark}

The above theorem follows directly from the finite-sample bound, which we state next.

\begin{theorem}
\label{thm:size-lower}
Under \Cref{assmp:indep,assmp:lower-score,assmp:lower-calib}, for any increasing sequence of integers, $\{m_k\}_{k\geq 1}$ such that $m_k\le k,\forall k\in\N$, whenever  $\mathbb{P}_{F_0,T,F_1}(\tau \ge T+m_T)>0$, we can bound  the conditional pre-change length of $\mathcal{C}_{low}^\alpha$ defined in \eqref{ci-lower} with $\hat{r}_t$ replaced by its oracle counterparts $r_t:=\P_{F_0,\infty}[\tau\geq t]$:
\begin{align*}
    &\mathbb{E}_{F_0,T,F_1}(|\mathcal{C}_{low}^\alpha\cap[T-1]|\mid \tau \ge T+m_T) \\
    &\le  \frac{1}{\mathbb{P}_{F_0,T,F_1}(\tau \ge T+m_T)} \left[ k_{T} - 1 + \sum_{k=k_{T}}^{T-1}  e^{
-\frac{m_k^2\delta^4}{32(C_vm_k+\frac{4}{k})}} \right],
\end{align*}
where $C_v=4 + 4(1+\log 2)^2, k_{T}=\inf\{k\in\N:m_k>8c_{\alpha,T}/\delta^2\}$ and $c_{\alpha,T}=-\log(\alpha r_{T-1})$.
\end{theorem}

The core proof idea of the above theorem is the following. To test whether a candidate point $t=T-k$ belongs to the confidence set, the algorithm evaluates the conformal martingale over a small post-change window of length $m_k$. Because this window contains true post-change data, the statistic accumulates a strictly positive, expected drift, i.e., $$\mathbb{E}[\log f_{T-k,T+l}(p^{T+l}_{T-k})] \ge \left( \frac{k\delta}{k+m_k}\right)^2,$$ for all $l=0,\cdots,m_k$.  Using boundedness and Lipschitz continuity of the logarithmic calibrator, we apply McDiarmid's inequality, proving that the probability of the test statistic failing to cross the required threshold decays exponentially with the window size.

Note that \Cref{thm:lower-asymp} follows from the above theorem by choosing $m_k=O(\log k)$. We rigorously establish all the proofs in \Cref{sec:app-proofs} in the Appendix Materials.


\subsection{Bound on post-change length for upper confidence set}
Similarly, one can derive a bound on the post-change length of the upper confidence set as well. The following assumption is analogous to \Cref{assmp:lower-calib} in the previous section.

\begin{assumption}
\label{assmp:upper-calib}
    For each $t$ and $j\leq t$, the calibrator used is $g_{t,j}(p)=1 + \lambda_{t,j}^*(p - 0.5)$ where $\lambda_{t,j}^* \in \arg\max_{\lambda \in [-1, 1]} \mathbb{E}\left[\log(1 + \lambda(q^{j}_{t} - 0.5))\right]$.
\end{assumption}

\begin{assumption}
\label{assmp:detector}
The change detector is such that $\tau$ under a change is stochastically smaller than $\tau$ under no change, i.e., $\P_{F_0,T,F_1}[\tau\geq t]\leq \P_{F_0,\infty}[\tau\geq t]$
for all $t,T\in\N$. 
\end{assumption}
We can expect all reasonable detection algorithms to satisfy the above, as they should stop later when a change is absent than when it is present. A similar assumption is also present in \cite{saha2025post}.

We begin with the asymptotic sharpness result, which is easier to interpret and follows directly from the finite-sample bound stated subsequently.

\begin{theorem}
\label{thm:upper-asymp}
 Fix any $s>\max\{\frac{8}{\delta^2},\frac{C_0}{\delta^4}\}$, where $C_0=160 + 128(1+\log 2)^2$. Consider a sequence of problems where the true changepoint $T$ increases to infinity, and the change detector also changes with $T$ in such a way that
$$R:=\liminf_{T\to\infty}\mathbb{P}_{F_0,T,F_1}(T\le\tau \le T+e^{T/s})>0.$$ Then, under \Cref{assmp:indep,assmp:lower-score,assmp:upper-calib,assmp:detector}, for the upper confidence set $\mathcal{C}_{\mathrm{up}}^\beta$ defined in \eqref{ci-upper}, with $\hat r_t$ replaced by its oracle counterpart $r_t:=\P_{F_0,\infty}[\tau\geq t]$, we have
$$\mathbb{E}_{F_0,T,F_1}(|\mathcal{C}_{up}^\beta\cap\{T,\cdots,\tau\}|\mid T\le\tau \le T+e^{T/s})
   = O(1). $$

\end{theorem}

\begin{remark}
The condition $\liminf_{T\to\infty}\mathbb{P}_{F_0,T,F_1}(T\le\tau \le T+e^{T/s})>0$
    means that the detector has a non-vanishing probability of stopping within an exponential window after the changepoint. If the delay were to exceed an exponential scale with high probability, the detector would be practically ineffective. As such, this is a mild and natural regularity condition that is expected to hold for any reasonable sequential change-point detection procedure. We required this assumption to ensure that there is a sufficient number of pre-change samples.
\end{remark}
We now state the corresponding finite-sample bound.
\begin{theorem}
\label{thm:size-upper}
Under \Cref{assmp:indep,assmp:lower-score,assmp:upper-calib,assmp:detector}, for any $c\in(0,\delta^2/8)$ and any increasing sequence of integers, $\{m_k\}_{k\geq 1}$ such that $m_k\ge k,\forall k\in\N$, whenever $R_{T}:=\mathbb{P}_{F_0,T,F_1}(T\le\tau \le T+m_T)>0$, we have the following bound on the conditional post-change length of $\mathcal{C}_{up}^\beta$ defined in \eqref{ci-upper} with $\hat{r}_t$ replaced by its oracle counterpart $r_t:=\P_{F_0,\infty}[\tau\geq t]$:
\begin{align*}
    &\mathbb{E}_{F_0,T,F_1}(|\mathcal{C}_{up}^\beta\cap\{T,\cdots,\tau\}|\mid T\le\tau \le T+m_T)\\
    &\le \frac{1}{R_{T}} \left[ k_1 +  \sum_{k=k_1}^{m_T} e^{ - \frac{\delta^4u_k^2}{32(C_vu_k + 4/k)}} + \sum_{k=k_1}^{m_T} e^{-cu_k}\right],
\end{align*}
where $u_k=\inf\{n\in\N: m_n\geq k\}$, $C_v=4+4(1+\log2)^2,$ and $k_1 = \inf\left\{k \in\N :  u_k >\frac{\log(1/\beta)}{\delta^2/8-c}\right\}$.
\end{theorem}

Note that \Cref{thm:upper-asymp} follows from the above theorem by choosing $m_k=\exp{(O(k))}$. We rigorously establish all the proofs in the appendix.

\subsection{Bound on the size of two-sided confidence set}
Finally, we combine the results of pre- and post-change parts to bound the two-sided length.

\begin{theorem}
\label{thm:twoside-asymp}
Under \Cref{assmp:indep,assmp:lower-score,assmp:lower-calib,assmp:upper-calib,assmp:detector}, fix any $s\in (0,\infty)$ and $s^\prime>\max\{\frac{8}{\delta^2},\frac{C_0}{\delta^4}\}$, where $C_0=160 + 128(1+\log 2)^2$. Consider a sequence of problems where the true changepoint $T$ increases to infinity, and the change detector also changes with $T$ in such a way that
$$\liminf_{T\to\infty}\mathbb{P}_{F_0,T,F_1}(T+s\log T\le\tau \le T+e^{T/s^\prime})>0.$$ Then, for the two-sided $\mathcal{C}$ in \eqref{ci-bothside-2alpha}, with $\hat r_t$ replaced by its oracle counterpart $r_t:=\P_{F_0,\infty}[\tau\geq t]$,
\begin{align*}
   &\mathbb{E}_{F_0,T,F_1}(|\mathcal{C}|\mid T+s\log T\le\tau \le T+e^{T/s^\prime}) \\
    &=\begin{cases}
    O(1), & s\delta^4>C_0\\
    O(\log T), & s\delta^4=C_0\\
    O(T^{1- s\delta^{4}/C_0}), &  s\delta^4<C_0.
\end{cases}
\end{align*}
Consequently, in all these three cases, $\mathbb{E}_{F_0,T,F_1}\left(\frac{|\mathcal{C}|}{T}\mid T+s\log T\le\tau \le T+e^{T/s^\prime}\right) =o(1)$.
\end{theorem}
\begin{remark}
The condition $\liminf_{T\to\infty}\mathbb{P}_{F_0,T,F_1}(T+s\log T\le\tau \le T+e^{T/s^\prime})>0$
    means that the detector has a non-vanishing probability of stopping within a window after the changepoint. This is natural in a doubly asymptotic regime where the detector threshold or ARL is allowed to scale with $T$.
\end{remark}

The proof follows by combining the arguments developed in the pre-change and post-change analyses and applying them to the respective segments of the confidence set.

\section{Learning the calibrator}
\label{sec:theory-real}
We now replace the oracle calibrator assumption in \Cref{assmp:lower-calib} with the following practical assumption. We also propose an algorithm for learning the calibrator and verify that it satisfies the assumption stated below.
\begin{assumption}
\label{assmp:lower-calib-real}
    For each $t$ and $j\geq t$, the calibrator used is $f_{t,j}(p)=1 + \lambda_{t,j}(p - 0.5)$ where $\{\lambda_{t,j}\}_{j\geq t}$ is learned using an online method such that the adaptive regret is logarithmic, i.e., there exists a constant $C^\prime>0$ such that for any $n\geq i\geq t$,
    \begin{align*}
        &\sum_{j=i}^n \log(1 + \lambda_{t,j}(p_{t}^j - 0.5)) \\
        &\geq \sup_{\lambda\in[-1,1]}\sum_{j=i}^n \log(1 + \lambda(p_{t}^j - 0.5))-C^\prime\log (n-t+1).
    \end{align*}
\end{assumption}
Online convex optimization algorithms, including online Newton step (ONS), exponentially weighted online optimization, and Cover's universal portfolio, together with the meta-algorithm \emph{Follow-The-Leading-History} (FLH) introduced by \cite{hazan2007adaptive}, satisfy this assumption. To be concrete, let us verify this assumption for ONS \citep{hazan2006logarithmic}, coupled with FLH \citep {hazan2007adaptive}, which is the algorithm used in our implementation as well. See \Cref{alg:flh-ons-calibrator} for an overview of the ONS+FLH procedure specialized to our setting. At a high level, FLH maintains a collection of ONS experts, where a new expert is initialized at every time step. Each expert independently runs ONS on the sequence of losses observed since its initialization, while FLH aggregates their predictions using exponentially weighted averaging. This design allows the procedure to adapt to changing environments and achieve logarithmic \emph{adaptive} regret by effectively competing with the best expert on every time interval. In our setting, the ONS experts are trained using the loss functions:
\[
\ell_{t,j}(\lambda)=-\log\!\left(1+\lambda(p_t^j-0.5)\right),\quad \lambda\in[-1,1],
\]
and a parameter $\eta$. A direct calculation gives
\[
\ell_{t,j}'(\lambda)=-\frac{p_t^j-0.5}{1+\lambda(p_t^j-0.5)},\quad\ell_{t,j}''(\lambda)
=\left(\ell_{t,j}'(\lambda)\right)^2.
\]
Hence, each $\ell_{t,j}$ is $1$-exp-concave. Moreover, since
\[
p_t^j-0.5\in[-1/2,1/2]\qquad\text{and}\qquad\lambda\in[-1,1],
\]
we have $|\ell_{t,j}'(\lambda)|=\left|\frac{p_t^j-0.5}{1+\lambda(p_t^j-0.5)}\right|\le 1.$
Therefore, the gradients over the set $[-1,1]$ are uniformly bounded by $1$.

Applying the static regret bound for ONS from Theorem 4 of \cite{hazan2006logarithmic}, with the gradient bound $G=1$, domain diameter
$D = \sup_{\lambda,\lambda'\in[-1,1]} |\lambda-\lambda'| = 2,$
the exp-concavity parameter $\gamma$ and the parameter of the algorithm, $\eta=\frac{1}{2}\min\{\gamma,1/(4GD)\}=1/16$
yields the static regret bound for any $n\geq t$,
\begin{align*}
    \mathcal{R}_n^t&:=\sum_{j=t}^n \ell_{t,j}(\lambda_{t,j})-\inf_{\lambda\in[-1,1]}\sum_{j=t}^n \ell_{t,j}(\lambda)\\
    &\le3\left(\frac{1}{\gamma}+4GD\right)\log(n-t+1)=27\log(n-t+1).
\end{align*}
Now, using Theorem 3.1 of \cite{hazan2007adaptive}  with the exp-concavity parameter $\gamma=1$, we get  the adaptive regret bound for any $n\geq i\geq t$,
\begin{align*}
    &\sum_{j=i}^n \ell_{t,j}(\lambda_{t,j})-\inf_{\lambda\in[-1,1]}\sum_{j=i}^n \ell_{t,j}(\lambda)\\
    &\le \sup_{[r,s]\subseteq[t,n]}\bigg|\sum_{j=r}^s \ell_{t,j}(\lambda_{t,j})-\inf_{\lambda\in[-1,1]}\sum_{j=r}^s \ell_{t,j}(\lambda)\bigg|\\
    &\le \mathcal{R}_n^t+O(\log (n-t+1)).
\end{align*}
Consequently, from the above two equations,  we conclude that \Cref{assmp:lower-calib-real} holds in our setting for \Cref{alg:flh-ons-calibrator} with $\eta=\frac{1}{16}$.

\begin{algorithm}[htbp]
\caption{Adaptively learning $\{\lambda_{t,j}\}_{j\geq t}$ (for any fixed $t\in\N$) via Follow-The-Leading-History (FLH) with Online Newton Step (ONS) experts}
\label{alg:flh-ons-calibrator}
\KwIn{ONS parameter $\eta>0$.}

\textbf{Initialize FLH:} $v_t^{(t)} \gets 1$\;
\textbf{Initialize first ONS expert:} $\lambda_t^{(t)} \gets 0$, $a_{t-1}^{(t)} \gets 0$, $b_{t-1}^{(t)} \gets 0$\;

\For{$j=t,t+1,\dots$}{
    Play the aggregated prediction: $\lambda_{t,j} \gets \sum_{i=t}^j v_j^{(i)} \lambda_j^{(i)}$\;
    Observe $p_t^j$ and define the loss function for the current step:
    \[
    \ell_j(\lambda) = -\log\!\left(1+\lambda(p_t^j-0.5)\right)
    \]
    Set $\hat v_{j+1}^{(j+1)} \gets 0$\;
    \For{$i = t,\dots,j$}{
        Compute intermediate unnormalized weights: 
        $ \hat v_{j+1}^{(i)} \gets \frac{v_j^{(i)} e^{- \ell_j(\lambda_j^{(i)})}}{\sum_{m=t}^j v_j^{(m)} e^{- \ell_j(\lambda_j^{(m)})}}$\;
    }
    Set weight for the new expert: $v_{j+1}^{(j+1)} \gets \frac{1}{k+1}$, where $k = j - t + 1$\;
    \For{$i = t,\dots,j$}{
        Scale down existing active experts: $v_{j+1}^{(i)} \gets \left(1-\frac{1}{k+1}\right) \hat v_{j+1}^{(i)}$\;
    }
    \For{$i = t,\dots,j$}{
        Compute gradient for expert $i$:
        \[
        g_j^{(i)} = \nabla \ell_j(\lambda_j^{(i)}) = -\frac{p_t^j-0.5}{1+\lambda_j^{(i)}(p_t^j-0.5)}
        \]
        Update ONS accumulators for expert $i$:
        \[
        a_j^{(i)} \gets a_{j-1}^{(i)} + (g_j^{(i)})^2,\]
        \[
        b_j^{(i)} \gets b_{j-1}^{(i)} + (g_j^{(i)})^2 \lambda_j^{(i)} - \frac{g_j^{(i)}}{\eta}
        \]
        Compute expert $i$'s prediction for step $j+1$:
        \[
        \lambda_{j+1}^{(i)} \gets \begin{cases}
            \min\left\{1,\max\left\{-1, \frac{b_j^{(i)}}{a_j^{(i)}}\right\}\right\},\hspace{-0.25cm}&\text{if } a_j^{(i)}\neq0\\
            0, &\text{otherwise.}\\
        \end{cases}
        \]
    }
    Start a new, $(j+1)$-th instance of ONS, which predicts: $\lambda_{j+1}^{(j+1)} \gets 0$\;
    Initialize its ONS variables: $a_j^{(j+1)} \gets 0, ~ b_j^{(j+1)} \gets 0$\;
}
\end{algorithm}

Now, we have the following result, which is similar to \Cref{thm:lower-asymp}, but with some adjustment due to the presence of a logarithmic regret incurred by the online learning procedure.

\begin{theorem}
\label{thm:lower-asymp-real}
 Fix any $s>8C^\prime/\delta^2$ and let $C_0=5+ 4(1+\log 2)^2$. Consider a sequence of problems where the changepoint $T \uparrow \infty$, and the change detector also changes with $T$ such that
$$L:=\liminf_{T\to\infty}\mathbb{P}_{F_0,T,F_1}(\tau \ge T+s\log T)>0.$$ Then, under \Cref{assmp:indep,assmp:lower-score,assmp:lower-calib-real}, for the lower confidence set $\mathcal{C}_{\mathrm{low}}^\alpha$ defined in \eqref{ci-lower}, with $\hat r_t$ replaced by its oracle counterpart $r_t:=\P_{F_0,\infty}[\tau\geq t]$, we have $\mathbb{E}_{F_0,T,F_1}(|\mathcal{C}_{low}^\alpha\cap[T-1]\mid \tau \ge T+s\log T) $
$$=\begin{cases}
    O(1), & 2(\delta^2s/8-C^\prime)^2>C_0s\\
    O(\log T), & 2(\delta^2s/8-C^\prime)^2=C_0s\\
    O\left(T^{1- \frac{2(\delta^2s/8-C^\prime)^2}{C_0s}}\right), &  2(\delta^2s/8-C^\prime)^2<C_0s.
\end{cases}$$
Consequently, in all these three cases, $\mathbb{E}_{F_0,T,F_1}\left(\frac{|\mathcal{C}_{low}^\alpha\cap[T-1]}{T-1}\mid \tau \ge T+s\log T\right) =o(1)$.
\end{theorem}
Note that by setting $C^\prime=0$ above, we recover the oracle result of \Cref{thm:lower-asymp}.

Similarly, one can derive a bound on the post-change length of the upper confidence set as well. The following assumption is analogous to \Cref{assmp:lower-calib-real}.
\begin{assumption}
\label{assmp:upper-calib-real}
    For each $t$ and $j\leq t$, the calibrator used is $g_{t,j}(p)=1 + \lambda_{t,j}(p - 0.5)$ where $\{\lambda_{t,j}\}_{j\leq t}$ is learned using an online method such that the regret is logarithmic, i.e., there exists a constant $C^\prime>0$ such that for any $t\geq n\geq i$,
    \begin{align*}
        &\sum_{j=i}^n \log(1 + \lambda_{t,j}(q_{t}^j - 0.5)) \\
        &\geq \max_{\lambda\in[-1,1]}\sum_{j=i}^n \log(1 + \lambda(q_{t}^j - 0.5))-C^\prime\log (n-i+1).
    \end{align*}
\end{assumption}
\begin{theorem}
\label{thm:upper-asymp-real}
Under \Cref{assmp:indep,assmp:lower-score,,assmp:detector,assmp:upper-calib-real}, fix any $s>\max\left\{\frac{16(2C^\prime+1)}{\delta^2},\frac{32C_0}{\delta^4}\right\}$, where $C_0=5+ 4(1+\log 2)^2$. Consider a sequence of problems where the true changepoint $T$ increases to infinity, and the change detector also changes with $T$ in such a way that
$$R:=\liminf_{T\to\infty}\mathbb{P}_{F_0,T,F_1}(T\le\tau \le T+e^{T/s})>0.$$ Then, for the upper confidence set $\mathcal{C}_{\mathrm{up}}^\beta$  in \eqref{ci-upper}, with $\hat r_t$ replaced by its oracle counterpart $r_t:=\P_{F_0,\infty}[\tau\geq t]$, we have
$$\mathbb{E}_{F_0,T,F_1}(|\mathcal{C}_{up}^\beta\cap\{T,\cdots,\tau\}\mid T\le\tau \le T+e^{T/s})
   = O(1). $$
\end{theorem}
Finally, we can combine the results of pre- and post-change parts and provide a bound on the two-sided length.

\begin{theorem}
\label{thm:twoside-asymp-real}
Under \Cref{assmp:indep,assmp:lower-score,assmp:detector,assmp:lower-calib-real,assmp:upper-calib-real}, fix any $s>\frac{8C^\prime}{\delta^2}$ and $s^\prime>\max\left\{\frac{16(2C^\prime+1)}{\delta^2},\frac{32C_0}{\delta^4}\right\}$, where $C_0=5 + 4(1+\log 2)^2$. Consider a sequence of problems where the true changepoint $T$ increases to infinity, and the change detector also changes with $T$ in such a way that
$$\liminf_{T\to\infty}\mathbb{P}_{F_0,T,F_1}(T+s\log T\le\tau \le T+e^{T/s^\prime})>0.$$ Then, for the two-sided confidence set $\mathcal{C}$ in \eqref{ci-bothside-2alpha}, with $\hat r_t$ replaced by its oracle counterpart $r_t:=\P_{F_0,\infty}[\tau\geq t]$, we have
$\mathbb{E}_{F_0,T,F_1}(|\mathcal{C}|\mid T+s\log T\le\tau \le T+e^{T/s^\prime})$
$$=\begin{cases}
    O(1), & 2(\delta^2s/8-C^\prime)^2>C_0s\\
    O(\log T), & 2(\delta^2s/8-C^\prime)^2=C_0s\\
    O\left(T^{1- \frac{2(\delta^2s/8-C^\prime)^2}{C_0s}}\right), &  2(\delta^2s/8-C^\prime)^2<C_0s.
\end{cases}$$
Consequently, in all these three cases, $\mathbb{E}_{F_0,T,F_1}\left(\frac{|\mathcal{C}|}{T}\mid T+s\log T\le\tau \le T+e^{T/s^\prime}\right) =o(1)$.
\end{theorem}
Note that by setting $C^\prime=0$ above, we recover the oracle result of \Cref{thm:twoside-asymp}.

The proof of the above theorem follows by combining the arguments developed in the pre-change and post-change analyses and applying them to the respective segments of the confidence set.

\section{Experimental results}
\label{sec:expt}
\subsection{Gaussian mean change}
We begin with a Gaussian mean change scenario.  Here, we draw $X_1,\cdots,X_{T-1}\stackrel{iid}{\sim}N(-1,1)$ and $X_T,X_{T+1},\cdots\stackrel{iid}{\sim}N(1,1)$, with $T=500$. We employ the  CUSUM detector based on the conformal test martingale with a threshold at $50000$. We use the score function $S(x)=x$. The functions $f_{t,j}(p)=1+\lambda_{t,j}(p-0.5)$ and $g_{t,j}(p)=1+\lambda_{t,j}^\prime(p-0.5)$ are used as a calibrator, with $\lambda_{t,j}, \lambda_{t,j}^\prime$ chosen adaptively via the ONS+FLH method described in \Cref{alg:flh-ons-calibrator}. And $\hat r_t$ is estimated by drawing $100$ sequences from the Uniform(0,1) distribution. Fig. \ref{fig:gaussian-ci} visualizes the two-sided confidence sets \eqref{ci-bothside-2alpha} across $4$ runs with $\alpha=0.05$. In Fig. \ref{fig:gaussian-teststat}, the average test statistics for the lower confidence set \eqref{ci-lower}, that is, the average value of $\max_{i:t\leq i\leq\tau}\prod_{j=t}^i f_{t,j}(p^j_t)$ and the average test statistic for the upper confidence set \eqref{ci-upper}, that is, the average value of $\max_{i:1\le i\le t-1}\prod_{j=i}^{t-1} g_{t-1,j}(q_{t-1}^j)$  are plotted (along with error bars) in blue and orange, respectively. We observe that regions where both averaged test statistics simultaneously attain low values are concentrated around the true changepoint $T=500$, demonstrating the effectiveness of the proposed method. Additional experiments in the Gaussian mean-shift setting, varying the changepoint location, detection threshold, and signal strength, are presented in \Cref{sec:app-expt} in the Appendix.

\begin{figure*}[!ht]
\centering
\centering
\subfloat[ Raw data is plotted over time till stopping. The two-sided confidence set \eqref{ci-bothside-2alpha} with $\alpha=0.05$ is marked in red points. Results of $4$ independent simulations are shown. \label{fig:gaussian-ci}]{\includegraphics[width=0.48\linewidth]{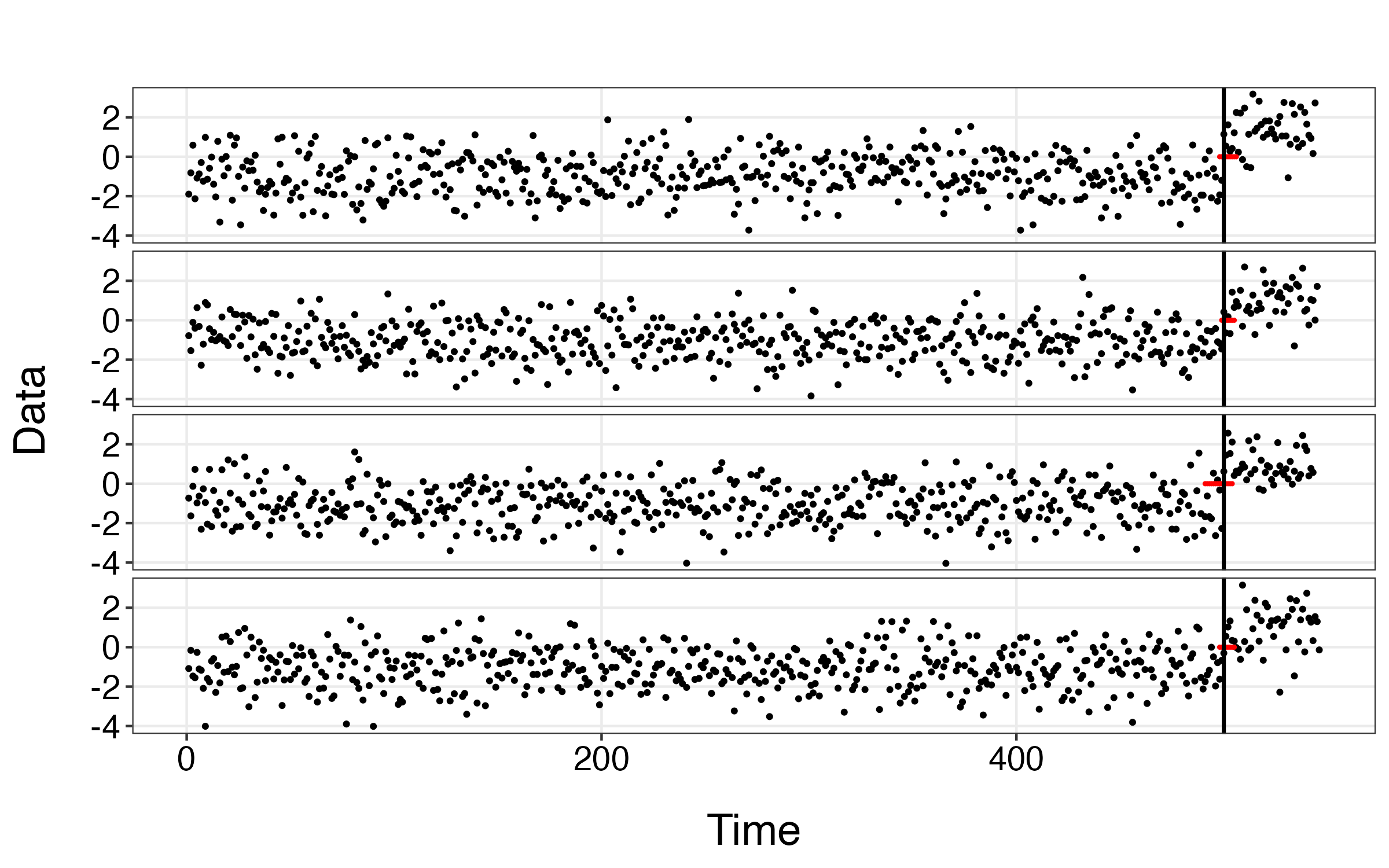}}
\quad
\subfloat[ Average (over $100$ simulations) test statistic for lower and upper confidence sets are plotted (with error bars). The horizontal black dashed line is the threshold for $95\%$ upper/lower set (or $90\%$ two-sided set). \label{fig:gaussian-teststat}]
{\includegraphics[width=0.48\linewidth]{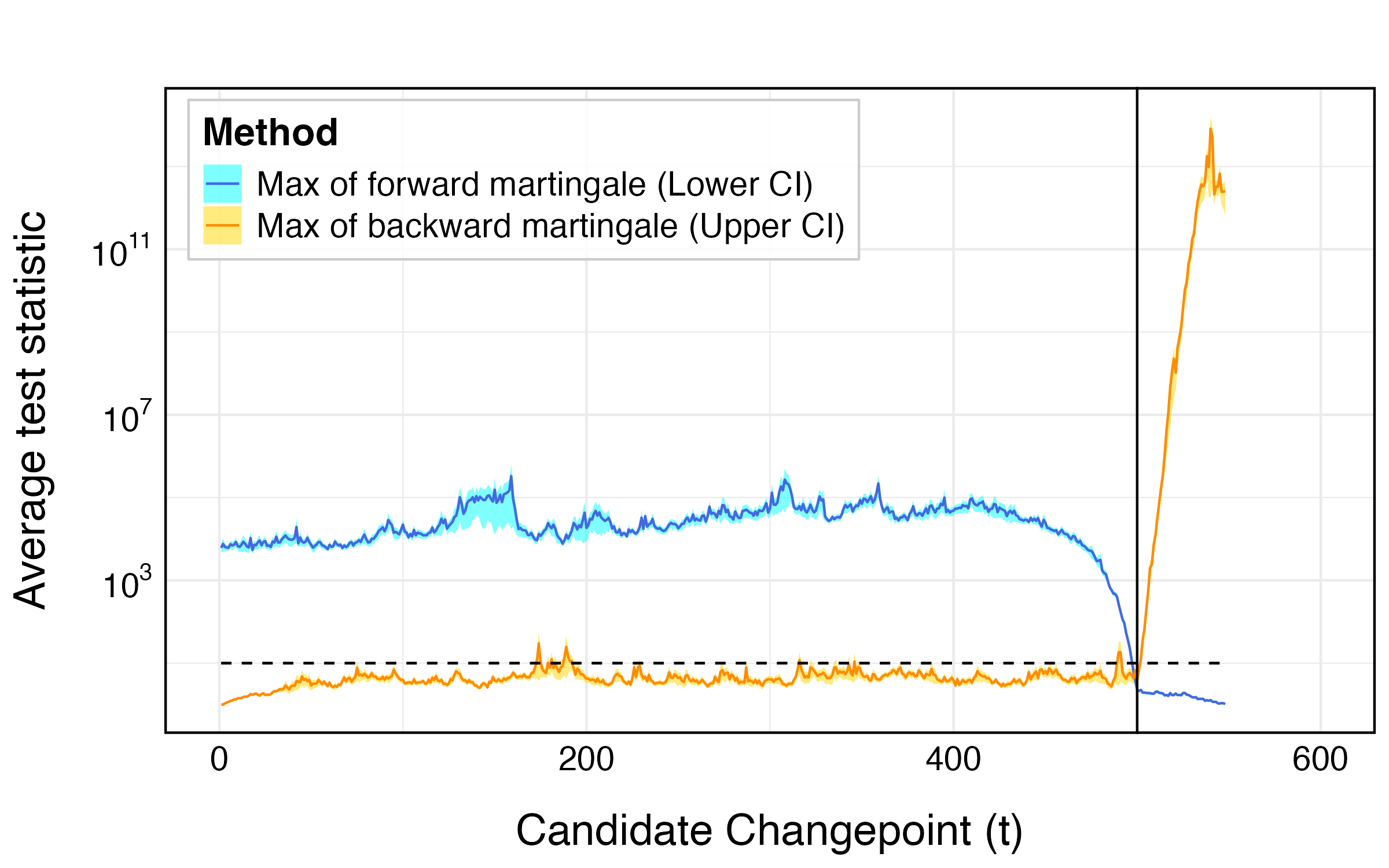}} 
\caption[]{Gaussian mean change: changepoint at $T=500$, marked in black vertical lines.} 
\label{fig:gaussian}  
\end{figure*}
\subsection{MNIST digit change experiment}
We consider a synthetic changepoint experiment on the MNIST dataset, following \cite{dandapanthula2025offline} but in the online setting. The true change point is at $T=400$.  We consider two different settings --- (i) Observations before the changepoint are handwritten digit 3, while observations after the changepoint are handwritten digit 7. 
(ii) 80\% of the observations before the changepoint are handwritten digit 3 (and the remaining 20\% are handwritten digit 7), while 20\% of the observations after the changepoint are handwritten digit 3 (and the remaining 80\% are handwritten digit 7).
We train a convolutional neural network classifier on MNIST and use the log-likelihood ratio style score $S(x)=\log\hat{P}(\text{Digit}=7|x)-\log\hat{P}(\text{Digit}=3|x)$, with $\alpha=\beta=0.05$. These scores are converted online into conformal p-values, which are calibrated to e-values using the  ONS+FLH strategy as in the previous experiment. We used a conformal test martingale-based CUSUM detector for change detection. The estimates
\(\hat r_t\) are
obtained via Monte Carlo simulation with i.i.d.\ uniform streams. The resulting two-sided sets are shown in \Cref{fig:pure} (Setting i) and in \Cref{fig:mixed} (Setting ii). For setting (i), the set is remarkably sharp:  [395, 400]. In the more challenging mixed-distribution case (setting ii), the confidence set becomes slightly larger, but remains quite sharp: $[396, 407]$. 
\begin{figure*}[!ht]
\centering
\centering
\subfloat[Pure digit change: $3$ to $7$. \label{fig:pure}]{\includegraphics[width=0.48\linewidth]{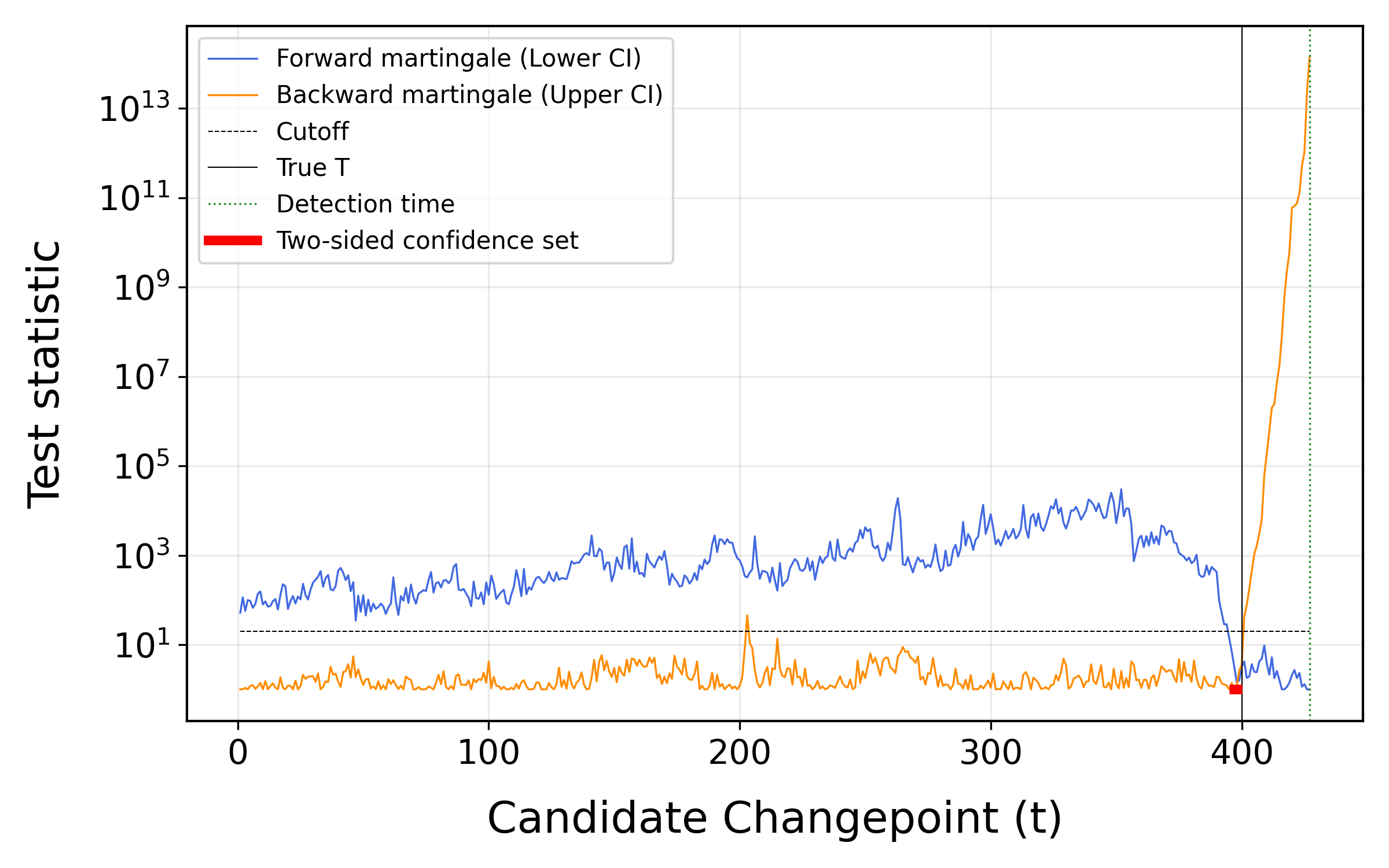}}
\quad
\subfloat[ 80\% digit 3 + 20\% digit 7 to
 20\% digit 3+ 80\% digit 7. \label{fig:mixed}]
{\includegraphics[width=0.49\linewidth]{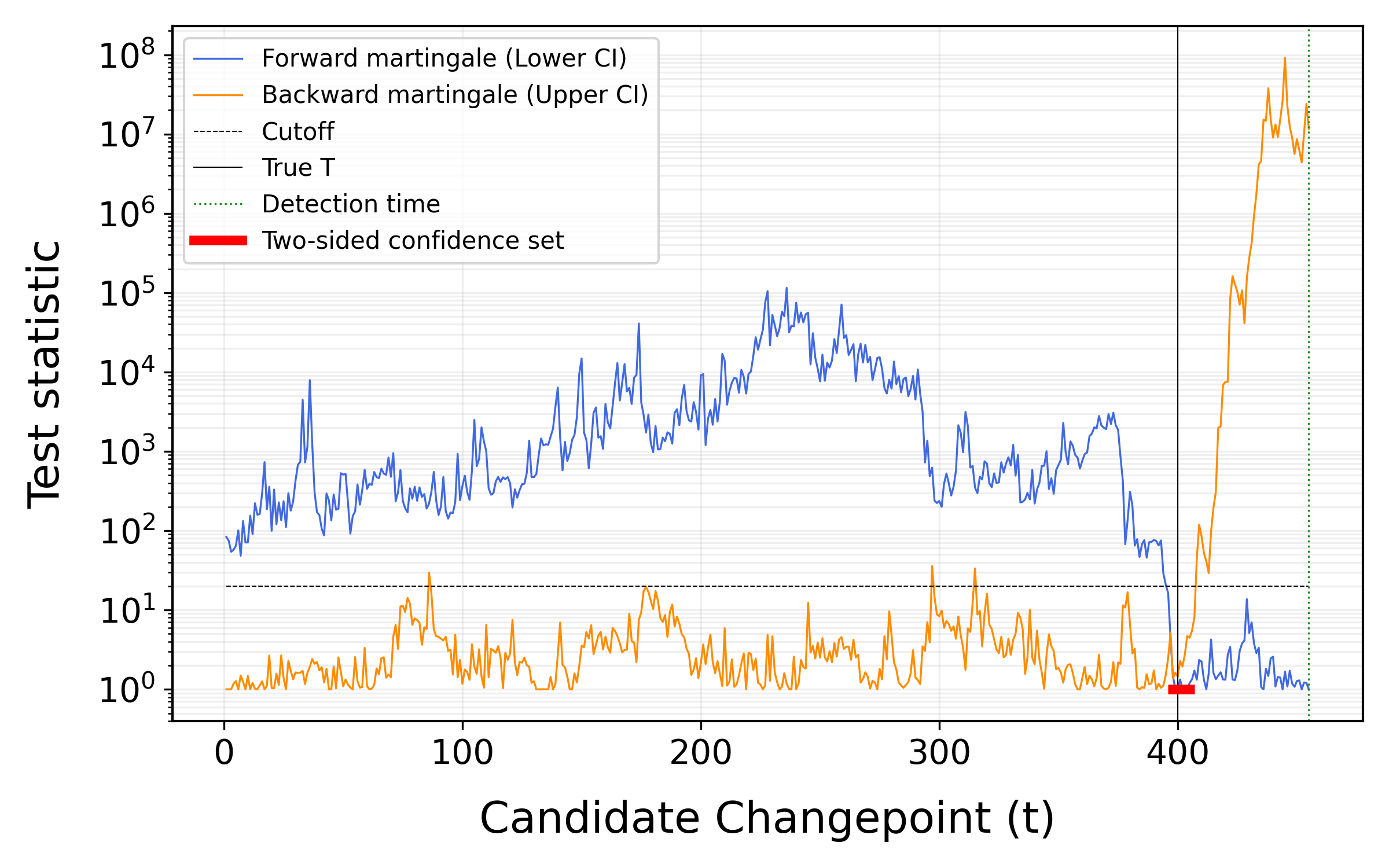}} 
\caption[]{Test statistics for lower and upper confidence sets are plotted, and the $90\%$ two-sided confidence set is shown in red.} 
\label{fig:mnist}  
\end{figure*}

\subsection{Wine quality dataset experiment}

Suppose a winery operates a production line dedicated to bottling \emph{pure white wines}.  An accidental switch in supply---for example, when red wine begins entering the bottling stream due to a logistics or labeling error---constitutes a serious quality-control failure.
Consequently, the monitoring system must determine how far back in time contamination may have started (after a detection), so that only white wines are released to customers.
To emulate this setting, we evaluate the lower post-detection confidence
set on the UCI Wine Quality dataset, following the white-to-red wine
change-point setup of \cite{vovk2021retrain}. The monitored sequence is constructed so that the first 399 observations are white wines and the remaining observations are red wines, giving a true changepoint of \(T=400\). Separate training samples from the white wine part, not included in the monitored sequence, are used to fit a linear regression model $\hat y_{\mathrm{white}}$ for predicting wine quality. 
For each monitored observation, we compute a fixed nonconformity score \(S(x)=|y-\hat y_{\mathrm{white}}(x)|\), so larger scores indicate observations that are less compatible with the pre-change white-wine model. A CUSUM detector based on a conformal test martingale is run. After detection, we implement our post-detection lower confidence set. For each candidate changepoint \(t\le \tau\), we compute forward conformal p-values using only the segment \(X_t,\ldots,X_j\) using the same pre-trained score function as used for detection, and transform it into a martingale using the Simple Jumper method. The lower confidence set \eqref{ci-lower} is then
computed using estimates
\(\hat r_t\),
obtained via Monte Carlo simulation with i.i.d.\ uniform streams.
For \(\alpha=0.1\), the resulting lower confidence set is
\([387,444] \cup [447, 449] \cup [451, 499]\). Thus, the lower confidence bound is $387$ (which is very close to the true changepoint  $400$), implying that all samples up to time $386$  can be certified as white wines with 90\% confidence.
\Cref{fig:wine} illustrates the test statistic and the resulting lower confidence set.
\begin{figure*}[!ht]
\centering
{\includegraphics[width=0.5\linewidth]{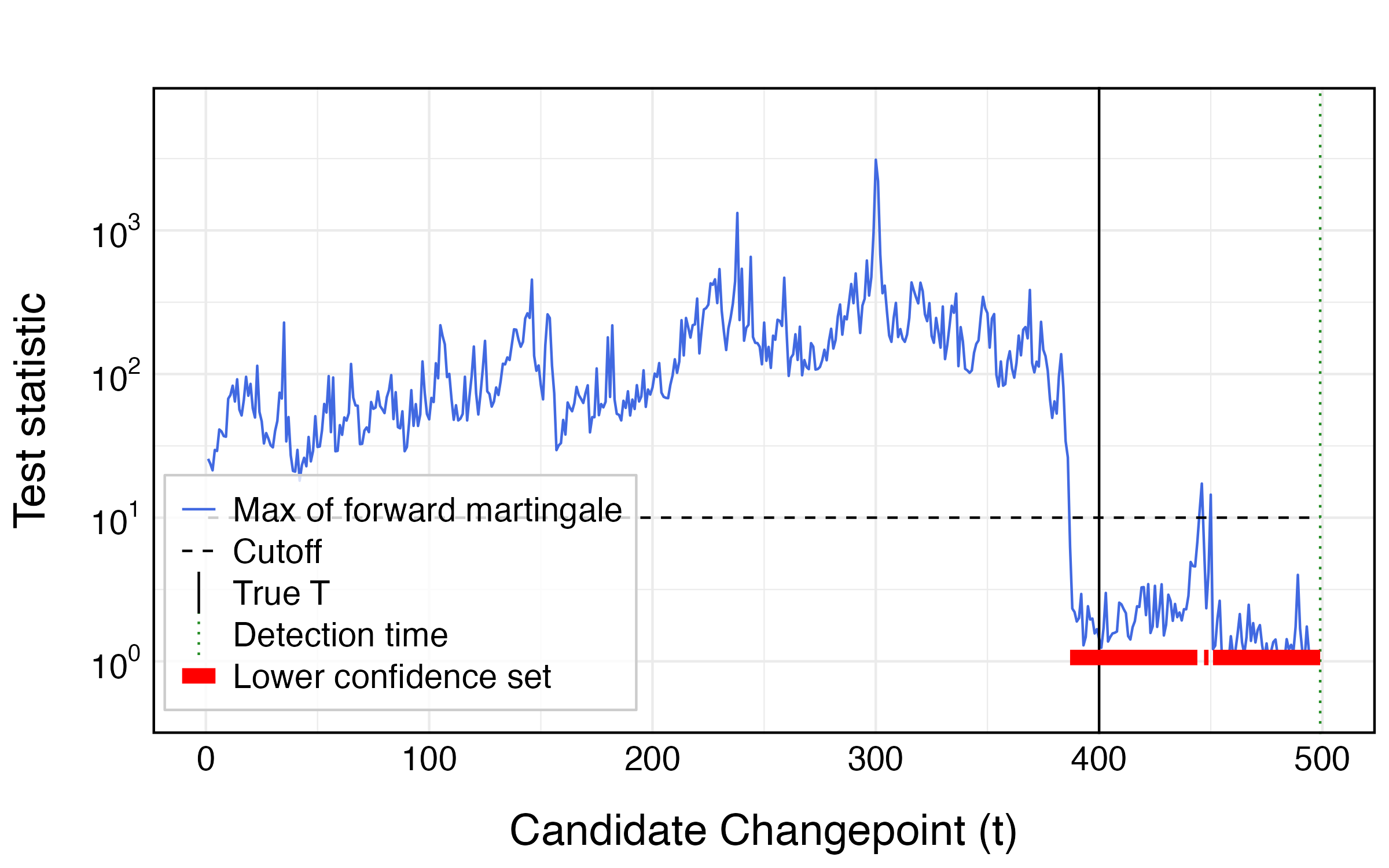}} 
\caption[]{The test statistic for lower confidence sets is plotted; the $90\%$ lower confidence set is in red. } 
\quad
\label{fig:wine}  
\end{figure*}

\section{Conclusion and future work}
\label{sec:conc}
We introduced the first distribution-free confidence sets for post-detection changepoint localization. 
We established rigorous non-asymptotic bounds on the size of these sets, demonstrating that reliable localization can be achieved without distributional assumptions. Furthermore, we prove that the conditional expected size of the confidence set remains uniformly bounded under suitable asymptotic regimes. Our approach is conceptually simple, broadly applicable, and exhibits strong empirical performance across both simulated and real datasets. One promising direction for future work is to extend our approach to a multi-stream setting. 

\bibliographystyle{apalike}
\bibliography{ref}

%% file: appendix-content.tex
\section{Omitted proofs}
\label{sec:app-proofs}

\subsection{Omitted proofs from \Cref{sec:method}}
\begin{proof}[Proof of \Cref{thm:cov-lower}] 
First, we bound the unconditional error at the true changepoint $T$. 
\begin{align*}
&\mathbb{P}_{F_0, T, F_1}(T \notin \mathcal{C}_{low}^\alpha, \tau \geq T) \\
&= \mathbb{E} \left[ \mathbb{P}_{F_0, T, F_1}\left(\sup_{T\leq n\leq \tau}\prod_{j=T}^n f(p_T^j)>\frac{1}{\alpha\hat r_T},\tau \geq T\mid  \hat{r}_T\right)\right] \\
&\leq \mathbb{E} \left[ \mathbb{P}_{F_0, T, F_1}\left(\sup_{n\geq T}\prod_{j=T}^n f(p_T^j)>\frac{1}{\alpha\hat r_T}\mid  \hat{r}_T\right)\right] \\
&\stackrel{(a)}{\leq} \mathbb{E}[\alpha \hat{r}_T] \\
&\leq \alpha \mathbb{P}_{F_0, \infty}(\tau \geq T).
\end{align*}
In step (a), we used Ville's inequality on the martingale $\{\prod_{j=T}^n f(p_T^j)\}_{n\geq T}$ and the fact that $\hat{r}_t$ is independent of the p-values. Remember that $\{\prod_{j=T}^n f(p_T^j)\}_{n\geq T}$ is a martingale because $\{X_n\}_{n\geq T}$ is exchangeable by \Cref{assmp-post}.

Next, observe that since $\tau$ is a stopping time, the event $\{\tau < T\}$ is $\mathcal{F}_{T-1}$-measurable. Because the distribution of $X_1, \dots, X_{T-1}$ is $F_0$ under both $\mathbb{P}_{F_0, T, F_1}$ and $\mathbb{P}_{F_0, \infty}$, we have $\mathbb{P}_{F_0, T, F_1}(\tau \geq T) = \mathbb{P}_{F_0, \infty}(\tau \geq T)$.
So,  $\mathbb{P}_{F_0, T, F_1}(T \notin \mathcal{C}_{low}^\alpha , \tau \geq T)\leq \alpha \mathbb{P}_{F_0, T, F_1}(\tau \geq T)$. Thus:
\begin{align*}
\mathbb{P}_{F_0, T, F_1}(T \notin \mathcal{C}_{low}^\alpha \mid \tau \geq T) &= \frac{\mathbb{P}_{F_0, T, F_1}(T \notin \mathcal{C}_{low}^\alpha, \tau \geq T)}{\mathbb{P}_{F_0, T, F_1}(\tau \geq T)}\\
&\leq \frac{\alpha \mathbb{P}_{F_0, T, F_1}(\tau \geq T)}{\mathbb{P}_{F_0, T, F_1}(\tau \geq T)} = \alpha.
\end{align*}
Taking the complement yields $\mathbb{P}_{F_0, T, F_1}(T \in \mathcal{C}_{low}^\alpha \mid \tau \geq T) \geq 1 - \alpha$.
\end{proof}

\begin{proof}[Proof of \Cref{thm:cov-upper}]
First, we bound the unconditional error at the true changepoint $T$. 
\begin{align*}
&\mathbb{P}_{F_0, T, F_1}(T \notin \mathcal{C}_{up}^\beta, \tau \geq T)\\
&= \mathbb{E} \left[ \mathbb{P}\left(\max_{i:1\leq i\leq T-1}\prod_{j=i}^{T-1} g_{T-1,j}(q^j_{T-1}) >\frac{1}{\beta\hat r_T},\tau \geq T\mid  \hat{r}_T\right)\right] \\
&\leq \mathbb{E} \left[ \mathbb{P}_{F_0, T, F_1}\left(\max_{i:1\leq i\leq T-1}\prod_{j=i}^{T-1} g_{T-1,j}(q^j_{T-1})  >\frac{1}{\beta\hat r_T}\mid  \hat{r}_T\right)\right] \\
&\stackrel{(a)}{\leq} \mathbb{E}[\beta \hat{r}_T] \\
&\leq \beta \mathbb{P}_{F_0, \infty}(\tau \geq T).
\end{align*}
In step (a), we used Ville's inequality on the martingale $\{\prod_{j=i}^{T-1} g_{T-1,j}(q^j_{T-1}))\}_{i< T}$ and the fact that $\hat{r}_t$ is independent of the p-values. Remember that $\{\prod_{j=i}^{T-1} g_{T-1,j}(q^j_{T-1})\}_{i< T}$ is a martingale because $\{X_n\}_{n< T}$ is exchangeable by \Cref{assmp-pre}.

The rest of the argument is the same as in the previous proof.
\end{proof}


\begin{proof}[Proof of \Cref{prop-rt}]
Under $T=\infty$, the observations are exchangeable by assumption. By validity of online conformal inference, the resulting $p$-values $(p_t)_{t\ge1}$ are i.i.d.\ $\mathrm{Uniform}(0,1)$ (and hence does not depend on $F_0$); see Theorem 8.2 of \cite{vovk2005algorithmic} for a proof. Since we assume that $\tau$ is a measurable function of $(p_t)_{t\ge1}$, its distribution depends only on the law of $(p_t)_{t\ge1}$, which does not depend on $F_0$.
\end{proof}

\subsection{Omitted proofs from \Cref{sec:theory}}
\begin{proof}[Proof of \Cref{thm:lower-asymp}]
    The proof follows by applying \Cref{thm:size-lower} with $m_k=\ceil{s\log k}$, for all $k$.  Because the logarithmic sequence $\ceil{s\log k}$ grows monotonically while $4/k$ strictly decreases to zero, there exists a finite critical index $k^\prime = \inf\left\{ k \ge 1 : \frac{4}{k} \le \ceil{s\log k} \right\}$ and hence
    \begin{align*}
        &\sum_{k=k_{T}}^{T-1}  \exp\!\left(
-\frac{m_k^2\delta^4}{32(C_vm_k+\frac{4}{k})}
\right)\\
&\leq\sum_{k=1}^{k^\prime}  \exp\!\left(
\frac{-m_k^2\delta^4}{32(C_vm_k+\frac{4}{k})}
\right) +\sum_{k=1}^{T-1}  e^{\frac{-s\delta^4\log k}{160 + 128(1+\log 2)^2}}\\
&=\sum_{k=1}^{k^\prime}  \exp\!\left(
-\frac{m_k^2\delta^4}{32(C_vm_k+\frac{4}{k})}
\right)+\sum_{k=1}^{T-1}k^{-\frac{s\delta^{4}}{C_0}}\\
&=\begin{cases}
    O(1), & s\delta^4>C_0\\
    O(\log T), & s\delta^4=C_0\\
    O(T^{1- s\delta^{4}/C_0}), &  s\delta^4<C_0.
\end{cases}
    \end{align*}
Since $r_{T-1}=\mathbb{P}_{F_0,T,F_1}(\tau \ge T-1)\geq\mathbb{P}_{F_0,T,F_1}(\tau \ge T+s\log T)$, we have
\begin{align*}
    \limsup_{T\to\infty}(c_{\alpha,T})&=-\log(\alpha\times\liminf_{T\to\infty} r_{T-1})\\
    &\leq-\log(\alpha\times\liminf_{T\to\infty} \mathbb{P}_{F_0,T,F_1}(\tau \ge T+s\log T))\\
    &\leq-\log(\alpha L),
\end{align*}
 for all $t$.
Note that here
\[k_{T}\leq\inf\{k\in\N:s\log k>8c_{\alpha,T}/\delta^2\}\leq 1+e^{\frac{8c_{\alpha,T}}{s\delta^2}},\]
and hence, $\limsup_{T\to\infty}k_T\leq1+e^{-\frac{8\log(\alpha L)}{s\delta^2}}.$
Thus, by applying \Cref{thm:size-lower},
\begin{align*}
   & \limsup_{T\to\infty}\mathbb{E}_{F_0,T,F_1}(|\mathcal{C}_{low}^\alpha\cap[T-1]\mid \tau \ge T+s\log T)\\
    &\leq \begin{cases}
    O(1), & s\delta^4>C_0\\
    O(\log T), & s\delta^4=C_0\\
    O(T^{1- s\delta^{4}/C_0}), &  s\delta^4<C_0.
\end{cases}
\end{align*}
\end{proof}
\begin{proof}[Proof of \Cref{thm:size-lower}]

Let $t < T$ be a candidate pre-change time, and denote its distance to the true changepoint by $k = T-t$. First, observe that for any pre-change evaluation step $t \le j < T$,  $\lambda_{t,j}^* = 0$, which implies $\log(f_{t,j}(p^j_t)) = 0$. Now, the logarithm of the test statistic at $t<T$ is
\[\max_{i:t\leq i\leq\tau}\sum_{j=t}^i \log(f_{t,j}(p^j_t))=\max_{i:t\leq i\leq\tau}\sum_{j=T}^i \log(f_{t,j}(p^j_t))\geq B_{k},\]
where $$B_{k}=\sum_{j=T}^{T+m_k-1}\log(f_{T-k,j}(p_{T-k}^j)).$$ Note that the inequality follows by dropping all but one term from the max, and just reindexing $t$ by $k = T-t$. Therefore,
\begin{align*}
    &\mathbb{E}_{F_0,T,F_1}(|\mathcal{C}_{low}^\alpha\cap[T-1]|\mid \tau\geq T+m_T)\\
    & =\sum_{t=1}^{T-1} \mathbb{P}_{F_0,T,F_1}(t \in \mathcal{C}_{low}^\alpha \mid \tau \ge T+m_T)\\
    &\le  \frac{\sum_{t=1}^{T-1} \mathbb{P}_{F_0,T,F_1}(B_{k} \le -\log(\alpha r_{T-k}))}{\mathbb{P}_{F_0,T,F_1}(\tau \ge T+m_T)}\\
    &\le  \frac{\sum_{k=1}^{T-1} \mathbb{P}_{F_0,T,F_1}(B_{k} \le c_{\alpha,T} )}{\mathbb{P}_{F_0,T,F_1}(\tau \ge T+m_T)},
\end{align*}
where $c_{\alpha,T}=-\log(\alpha r_{T-1})$ because $r_t=\mathbb{P}_{F_0,\infty}(\tau \ge t)$ is non-increasing with respect to $t$.

Let, $\mu_k=\mathbb{E}(B_{k})$. From Lemma \ref{lem:kl}, we have
\begin{equation}
\label{eq:mu_k}
    \mu_k=\sum_{l=0}^{ m_k - 1} \left( \frac{k\delta}{k+l+1} \right)^2 \ge \left( \frac{k\delta}{k+m_k} \right)^2  m_k \geq \delta^2m_k/4,
\end{equation}
where the last inequality follows from the fact that $m_k\leq k$.
 Define, $k_{T}=\inf\{k:m_k>8c_{\alpha,T}/\delta^2\}$.  In order to bound the sum $\sum_{k=1}^{T-1} \mathbb{P}_{F_0,T,F_1}(B_{k} \le c_{\alpha,T} )$, we now partition it into two regimes.

Regime 1: ($1 \le k < k_{T}$). For these candidates, we trivially bound their probability of retention by $1$.$$\sum_{k=1}^{k_{T}-1} \mathbb{P}_{F_0,T,F_1}(B_{k} \le c_{\alpha,T}) \le k_{T} - 1.$$
Regime 2: ($k \ge k_{T}$). For these candidates, $\mu_k-c_{\alpha,T} > \frac{\delta^2m_k}{8}$. Also, note that $\log f_{t,j}$ is Lipschitz with 
$$L=\max_{p\in[0,1]} \Bigg|\frac{f_{t,j}'(p)}{f_{t,j}(p)}\Bigg| = \frac{\displaystyle\max_{\lambda\in[-1,1]} |\lambda|}{\displaystyle\min_{p\in[0,1], \lambda\in[-1,1]} |1 + \lambda(p - 0.5)|} = 2.$$
Moreover, $\log f_{t,j}(p) \in [\log(0.5), \log(1.5)]$ and hence its absolute value is universally bounded by $M=\log 2$.
Now, we employ \Cref{lem:bound-mcd} with $M=\log 2$ and $L=2$ to obtain
\begin{align*}
    P\!\left(
B_k \le c_{\alpha,T}
\right)
&\le
\exp\!\left(
-\frac{2(\mu_k-c_{\alpha,T})^2}{C_vm_k+4/k}
\right)\\
&\le \exp\!\left(
-\frac{m_k^2\delta^4}{32(C_vm_k+4/k)}
\right).
\end{align*}
Therefore, combining the bounds from regime 1 and regime 2, we derive the final bound that  $\mathbb{E}_{F_0,T,F_1}(|\mathcal{C}_{low}^\alpha\cap[T-1]|\mid \tau \ge T+m_T)$ is less than or equal to
$$  \frac{k_{T} - 1 + \sum_{k=k_{T}}^{T-1}  \exp\!\left(
-\frac{m_k^2\delta^4}{32(C_vm_k+4/k)}
\right)}{\mathbb{P}_{F_0,T,F_1}(\tau \ge T+m_T)}.$$

\end{proof}

\begin{proof}[Proof of \Cref{thm:upper-asymp}]
    Choose $K_0$ large enough so that $\floor{e^{k/s}}\geq k$ for all $k\geq K_0$. Define
    \[
        m_k =
        \begin{cases}
            k, & k<K_0,\\
            \floor{e^{k/s}}, & k\geq K_0.
        \end{cases}
    \]
    Then $m_k\geq k$ for all $k$, and for all sufficiently large $T$, $m_T=\floor{e^{T/s}}$. Thus the conditioning event $T\leq \tau\leq T+m_T$ agrees with $T\leq \tau\leq T+e^{T/s}$ up to integer rounding for all sufficiently large $T$.

    The proof now follows by applying \Cref{thm:size-upper} with this choice of $m_k$. Let
    \[
        u_k=\inf\{n\in\mathbb{N}:m_n\geq k\}.
    \]
    Since $m_n=\floor{e^{n/s}}$ for all large $n$, there exists a finite index $k^\prime$ such that, for all $k\geq k^\prime$,
    \[
        u_k\geq s\log k
        \qquad\text{and}\qquad
        \frac{4}{k}\leq u_k .
    \]
    Therefore,
    \begin{align*}
      &  \sum_{k=k_{1}}^{m_T}  \exp\!\left(
-\frac{u_k^2\delta^4}{32(C_vu_k+\frac{4}{k})}
\right)\\
&\leq \sum_{k=1}^{\infty}  \exp\!\left(
-\frac{u_k^2\delta^4}{32(C_vu_k+\frac{4}{k})}
\right)\\
&\leq\sum_{k=1}^{k^\prime-1}  \exp\!\left(
-\frac{u_k^2\delta^4}{32(C_vu_k+\frac{4}{k})}
\right)\\
&\qquad
+\sum_{k=k^\prime}^{\infty}  \exp\!\left(\frac{-s\delta^4\log k}{160 + 128(1+\log 2)^2}
\right)\\
&=\sum_{k=1}^{k^\prime-1}  \exp\!\left(
-\frac{u_k^2\delta^4}{32(C_vu_k+\frac{4}{k})}
\right)+\sum_{k=k^\prime}^{\infty}k^{-\frac{s\delta^4}{160 + 128(1+\log 2)^2}}\\
&<\infty,
\end{align*}
since $\frac{s\delta^4}{160 + 128(1+\log 2)^2}>1$. Now, choose $c\in(1/s,\delta^2/8)$ (note that it is possible since $s>8/\delta^2$). Then,
\begin{align*}
    \sum_{k=k_1}^{m_T} e^{-cu_k}
    &\leq \sum_{k=1}^{k^\prime-1}e^{-cu_k}+\sum_{k=k^\prime}^{\infty}k^{-cs}<\infty,
\end{align*}
since $cs>1$. Thus, by applying \Cref{thm:size-upper},
\begin{align*}
&\limsup_{T\to\infty}\mathbb{E}_{F_0,T,F_1}(|\mathcal{C}_{up}^\beta\cap\{T,\cdots,\tau\}|\mid T\le\tau \le T+e^{T/s})\\
&\leq \frac{1}{R}\left(k_1+ \sum_{k=1}^{\infty}  e^{
-\frac{u_k^2\delta^4}{32(C_vu_k+\frac{4}{k})}}+\sum_{k=k_1}^{\infty} e^{-cu_k}\right)<\infty.
\end{align*}
\end{proof}

\begin{proof}[Proof of \Cref{thm:size-upper}]

Let $t \geq T$ be a candidate post-change time, and denote its distance to the true changepoint by $k = t-T$. Since the upper confidence set for a candidate $t$ uses the backward martingale ending at time $t-1$, we treat the case $k=0$ separately and bound its contribution by $1$. Thus, in the rest of the proof, assume $k\ge 1$ and write $t=T+k$.

First, observe that for any $j$ such that $T+k-1\ge j \ge T$, $\lambda_{T+k-1,j}^* = 0$, which implies $\log(g_{T+k-1,j}(q^j_{T+k-1})) = 0$. Now, the logarithm of the test statistic at  $t=T+k$ is
\begin{align*}
& \max_{i:1\leq i\leq T+k-1}\sum_{j=i}^{T+k-1} \log(g_{T+k-1,j}(q^j_{T+k-1}))\\
&=
\max_{i:1\leq i\leq T+k-1}\sum_{j=i}^{T-1} \log(g_{T+k-1,j}(q^j_{T+k-1}))
\geq B_{k},
\end{align*}
where
\[
B_{k}=\sum_{j=T-u_k}^{T-1}\log(g_{T+k-1,j}(q_{T+k-1}^j)),
\]
where $u_k=\inf\{n\in\N: m_n\geq k\}$. Note that the inequality follows by dropping all but one term from the max, and just reindexing $t$ by $k = t-T$. Since $m_k\geq k,$ we have $u_k\leq k$.

Therefore,
\begin{align*}
   &\mathbb{E}_{F_0,T,F_1}(|\mathcal{C}_{up}\cap\{T,\cdots,\tau\}\mid T\le\tau \le T+m_T)\\
   & \leq\sum_{t=T}^{T+m_T} \mathbb{P}_{F_0,T,F_1}(t \in \mathcal{C}_{up} \mid T\le\tau \le T+m_T)\\
    &\le  \frac{1+\sum_{k=1}^{m_T} \mathbb{P}_{F_0,T,F_1}(B_{k} \le -\log(\beta r_{T+k}),  \tau\geq T+k)}{\mathbb{P}_{F_0,T,F_1}(T\le\tau \le T+m_T)}\\
    &\le  \frac{1+\sum_{k=1}^{m_T} \mathbb{P}_{F_0,T,F_1}(B_{k} \le -\log(\beta r_{\tau}) )}{\mathbb{P}_{F_0,T,F_1}(T\le\tau \le T+m_T)},
\end{align*}
 because $r_t=\mathbb{P}_{F_0,\infty}(\tau \ge t)$ is non-increasing with respect to $t$, and on the event $\{\tau\ge T+k\}$ we have $r_{T+k}\ge r_\tau$.
Let $\mu_k=\mathbb{E}(B_{k})$. From \Cref{lem:kl}, we have
\begin{equation}
\label{eq:mu_k_upper}
\mu_k \ge \sum_{l=0}^{u_k - 1} \left( \frac{k\delta}{k+l+1} \right)^2 \ge  \frac{\delta^2 u_k}{4},
\end{equation}
where the last inequality follows from $u_k\le k$. Now, we use the fact that $r_\tau$ is super-uniform (see \Cref{lem:rtau} for a proof) and obtain, for any $c\in(0,\delta^2/8)$,
\begin{align*}
    &\mathbb{P}_{F_0,T,F_1}(B_{k} \le -\log(\beta r_{\tau}) )\\
    &=\mathbb{P}_{F_0,T,F_1}(B_{k} \le -\log(\beta r_{\tau}),r_{\tau}\geq e^{-cu_k} )\\
    &\qquad+\mathbb{P}_{F_0,T,F_1}(r_{\tau}< e^{-cu_k})\\
    &\leq \mathbb{P}_{F_0,T,F_1}(B_{k} \le cu_k  -\log(\beta))+ e^{-cu_k}.
\end{align*}
Define
\[
k_1 = \inf\left\{k \ge 1 : u_k >\frac{\log(1/\beta)}{\delta^2/8-c} \right\}.
\]
In order to bound the sum $\sum_{k=1}^{m_T} \mathbb{P}_{F_0,T,F_1}(B_{k} \le cu_k  -\log(\beta))$, we partition it into two regimes.

Regime 1: ($1 \le k < k_1$). For these candidates near the true changepoint, the expected drift may not yet overcome the threshold. We trivially bound the probability by $1$:
\[
\sum_{k=1}^{k_1-1} \mathbb{P}_{F_0,T,F_1}(B_{k} \le cu_k  -\log(\beta)) \le k_1-1.
\]

Regime 2: ($k \ge k_1$). For these candidates,
\[
\mu_k - cu_k+\log\beta > \frac{\delta^2}{8}u_k.
\]
We employ \Cref{lem:bound-mcd} applied to the backward sum $B_k$, which is sum of $u_k$ many terms and with $L=2$, $M=\log 2$ analogous to the forward case. Hence,
\[
\mathbb{P}_{F_0,T,F_1}(B_k \le  cu_k  -\log\beta) \le
\exp\left\{\frac{-2(\mu_k  -cu_k  +\log\beta)^2}{C_vu_k + 4/k}\right\}
\]
where $C_v=4+4(1+\log2)^2$.
Summing this bound over Regime 2 up to the finite window limit $m_T$, and utilizing the explicit lower bound for $\mu_k$ derived in \eqref{eq:mu_k_upper}, we get
\[
\sum_{k=k_1}^{m_T} \mathbb{P}_{F_0,T,F_1}(B_{k} \le cu_k  -\log\beta)
\le
\sum_{k=k_1}^{m_T}
\exp\left\{
\frac{-\delta^4u_k^2}{32(C_vu_k + \frac{4}{k})}
\right\}.
\]
Therefore, combining the bounds from Regime 1 and Regime 2, we derive the final bound that
\[
\mathbb{E}_{F_0,T,F_1}(|\mathcal{C}_{up}\cap\{T,\dots,\tau\}\mid T\le\tau \le T+m_T)
\]
is less than or equal to
\[
 \frac{
 k_1
 + \sum_{k=k_1}^{m_T}  \exp\left\{ - \frac{\delta^4u_k^2}{32(C_vu_k + 4/k)}\right\}
 + \sum_{k=k_1}^{m_T} e^{-cu_k}
 }{
 \mathbb{P}_{F_0,T,F_1}(T\le\tau \le T+m_T)
 } .
\]

\end{proof}

\subsection{Omitted proofs from \Cref{sec:theory-real}}
\begin{proof}[Proof of \Cref{thm:lower-asymp-real}]
Let $m_n=\ceil{s\log n}$ for all $n\in\N$ and $t < T$ be a candidate pre-change time, and denote its distance to the true changepoint by $k = T-t$. Now, the logarithm of the test statistic at $t<T$ is
\[\max_{i:t\leq i\leq\tau}\sum_{j=t}^i \log(f_{t,j}(p^j_t))\geq B_{k}^{(0)}+B_{k}^{(1)},\]
where $$B_{k}^{(0)}=\sum_{j=T-k}^{T-1}\log(f_{T-k,j}(p_{T-k}^j)),$$
$$B_{k}^{(1)}=\sum_{j=T}^{T+m_k-1}\log(f_{T-k,j}(p_{T-k}^j)).$$ Note that the inequality follows by dropping all but one term (corresponding to $i=T+m_k-1$) from the max, breaking the sum $\sum_{j=t}^{T+m_k-1}\log(f_{T-k,j}(p_{T-k}^j))$ into two parts, and just reindexing $t$ by $k = T-t$. Now, by \Cref{assmp:lower-calib-real},
\begin{align*}
   & B_{k}^{(0)}=\sum_{j=T-k}^{T-1}\log(1+\lambda_{t,j}(p_{T-k}^j-0.5))\\
   &\geq \max_{\lambda\in[-1,1]} \sum_{j=T-k}^{T-1}\log(1+\lambda(p_{T-k}^j-0.5))-C^\prime\log k\\
    &\geq -C^\prime\log k,
\end{align*}
    where the last inequality follows from the fact that $\sum_{j=T-k}^{T-1}\log(1+\lambda(p_{T-k}^j-0.5))=0$ for $\lambda=0$. Again, by \Cref{assmp:lower-calib-real},
\begin{align*}
   & B_{k}^{(1)}=\sum_{j=T}^{T+m_k-1}\log(1+\lambda_{t,j}(p_{T-k}^j-0.5))\\
    &\geq \max_{\lambda\in[-1,1]} \sum_{j=T}^{T+m_k-1}\log(1+\lambda(p_{T-k}^j-0.5))-C^\prime\log(k+m_k)\\
    &\geq B_{k}^{(2)}-C^\prime\log (2k),
\end{align*}
where $B_{k}^{(2)}=\sum_{j=T}^{T+m_k-1}\log(1+\lambda^{**}_k(p_{T-k}^j-0.5))$ $\lambda^{**}_k$ is as defined in \eqref{eq:lammbda-opt}. Therefore,
\begin{align*}
   & \mathbb{E}_{F_0,T,F_1}(|\mathcal{C}_{low}^\alpha\cap[T-1]|\mid \tau\geq T+m_T)\\
   & =\sum_{t=1}^{T-1} \mathbb{P}_{F_0,T,F_1}(t \in \mathcal{C}_{low}^\alpha \mid \tau \ge T+m_T)\\
    &\le  \frac{\sum_{t=1}^{T-1} \mathbb{P}_{F_0,T,F_1}(B_{k}^{(0)}+B_{k}^{(1)} \le -\log(\alpha r_{T-k}))}{\mathbb{P}_{F_0,T,F_1}(\tau \ge T+m_T)}\\
    &\le  \frac{\sum_{k=1}^{T-1} \mathbb{P}_{F_0,T,F_1}(B_{k}^{(2)} \le c_{\alpha,T}+ 2C^\prime\log k+C^\prime\log 2 )}{\mathbb{P}_{F_0,T,F_1}(\tau \ge T+m_T)},
\end{align*}
where $c_{\alpha,T}=-\log(\alpha r_{T-1})$ because $r_t=\mathbb{P}_{F_0,\infty}(\tau \ge t)$ is non-increasing with respect to $t$.

Let, $\mu_k=\mathbb{E}(B_{k}^{(2)})$. From the lemma \ref{lem:kl}, we have
\begin{equation}
\label{eq:mu_k-real}
    \mu_k \ge \left( \frac{k\delta}{k+m_k} \right)^2  m_k \geq \frac{\delta^2m_k}{4},
\end{equation}
where the last inequality follows from the fact that $m_k\leq k$.
 Define, $k_{T}=\inf\{k\in\N:(\delta^2s/8-C^\prime)\log k>C^\prime\log2+c_{\alpha,T}\}\leq1+\exp{(\frac{C^\prime\log2+c_{\alpha,T}}{\delta^2s/8-C^\prime})}$.  In order to bound the sum $\sum_{k=1}^{T-1} \mathbb{P}_{F_0,T,F_1}(B_{k} \le c_{\alpha,T} )$, we now partition it into two regimes.

Regime 1: ($1 \le k < k_{T}$). For these candidates, we trivially bound their probability of retention by $1$.
\begin{align*}
    \sum_{k=1}^{k_{T}-1} \mathbb{P}_{F_0,T,F_1}(B_{k}^{(2)} &\le c_{\alpha,T}+2C^\prime\log k+C^\prime\log 2) \\
    &\le k_{T} - 1\leq\exp{\left(\frac{C^\prime\log 2+c_{\alpha,T}}{\delta^2s/8-C^\prime}\right)}.
\end{align*}
Regime 2: ($k \ge k_{T}$). For these candidates, $\mu_k-2C^\prime\log k-C^\prime\log 2-c_{\alpha,T} > (\delta^2s/8-C^\prime)\log k$. Also, note that $\log f_{t,j}$ is Lipschitz with 
$$L=\max_{p\in[0,1]} \Bigg|\frac{f_{t,j}'(p)}{f_{t,j}(p)}\Bigg| = \frac{\displaystyle\max_{\lambda\in[-1,1]} |\lambda|}{\displaystyle\min_{p\in[0,1], \lambda\in[-1,1]} |1 + \lambda(p - 0.5)|} = 2.$$
Moreover, $\log f_{t,j}(p) \in [\log(0.5), \log(1.5)]$ and hence its absolute value is universally bounded by $M=\log 2$.
Now, we employ \Cref{lem:bound-mcd} with $M=\log 2$ and $L=2$ to obtain
\begin{align*}
   & P\!\left(
B_{k}^{(2)} \le c_{\alpha,T}+ 2C^\prime\log k+C^\prime\log 2
\right)\\
&
\le
\exp\!\left(
-\frac{2(\mu_k-2C^\prime\log k-C^\prime\log 2-c_{\alpha,T})^2}{C_vm_k+4/k}
\right)\\
&\le \exp\!\left(
-\frac{2((\delta^2s/8-C^\prime)\log k)^2}{C_v s\log k+4/k}
\right).
\end{align*}
 Because the logarithmic sequence $s \log k$ grows monotonically while $4/k$ strictly decreases to zero, there exists a finite critical index $k^\prime = \inf\left\{ k \ge 1 : \frac{4}{k} \le \ceil{s \log k} \right\}$ and hence
\begin{align*}
&\sum_{k=k_{T}}^{T-1}   \exp\!\left(
-\frac{2((\delta^2s/8-C^\prime)\log k)^2}{C_v s\log k+4/k}
\right)\\
&\leq\sum_{k=1}^{k^\prime}   \exp\!\left(
-\frac{2((\delta^2s/8-C^\prime)\log k)^2}{C_v s\log k+4/k}
\right)\\
&\qquad+\sum_{k=1}^{T-1}  \exp\!\left(-\frac{2(\delta^2s/8-C^\prime)^2\log k}{(C_v+1)s}
\right)\\
&=\sum_{k=1}^{k^\prime}  \exp\!\left(
-\frac{2((\delta^2s/8-C^\prime)\log k)^2}{C_v s\log k+4/k}
\right)+\sum_{k=1}^{T-1}k^{-\frac{2(\delta^2s/8-C^\prime)^2}{(C_v+1)s}}\\
&=\begin{cases}
    O(1), & 2(\delta^2s/8-C^\prime)^2>(C_v+1)s\\
    O(\log T), & 2(\delta^2s/8-C^\prime)^2=(C_v+1)s\\
    O\left(T^{1- \frac{2(\delta^2s/8-C^\prime)^2}{(C_v+1)s}}\right), &  2(\delta^2s/8-C^\prime)^2<(C_v+1)s.
\end{cases}
\end{align*}
Since $r_{T-1}=\mathbb{P}_{F_0,T,F_1}(\tau \ge T-1)\geq\mathbb{P}_{F_0,T,F_1}(\tau \ge T+s\log T)$, we have
$\limsup_{T\to\infty}(c_{\alpha,T})=-\log(\alpha\times\liminf_{T\to\infty} r_{T-1})\leq-\log(\alpha\times\liminf_{T\to\infty} \mathbb{P}_{F_0,T,F_1}(\tau \ge T+s\log T))\leq-\log(\alpha L),$
 for all $t$.
Note that here
 $\limsup_{T\to\infty}k_T\leq1+\limsup_{T\to\infty}\exp{(\frac{C^\prime\log2-\log(\alpha L)}{\delta^2s/8-C^\prime})}\leq1+\exp{(\frac{C^\prime\log2+c_{\alpha,T}}{\delta^2s/8-C^\prime})}<\infty.$
Thus, we have proved the desired result.
\end{proof}

\begin{proof}[Proof of \Cref{thm:upper-asymp-real}]
Let $m_n=\floor{e^{n/s}}$ for all $n\in\N$. We shall prove the result for the
conditioning event $T\le \tau \le T+m_T$. Since
\[
\liminf_{T\to\infty}\mathbb{P}_{F_0,T,F_1}(T\le \tau \le T+m_T)>0,
\]
it suffices to show that the numerator in the conditional expectation is
uniformly bounded.

Let $t\ge T$ be a candidate post-change time and write $k=t-T$. Recall that,
for the upper confidence set, candidate $t$ is tested using the backward
martingale ending at time $t-1$. Thus, for $k\ge 1$, the logarithm of the test
statistic for candidate $t=T+k$ satisfies
\[
\max_{i:1\leq i\leq T+k-1}\sum_{j=i}^{T+k-1}
\log(g_{T+k-1,j}(q^j_{T+k-1}))
\geq B_k^{(0)}+B_k^{(1)},
\]
where
\[
B_k^{(0)}
=\sum_{j=T-u_k}^{T-1}
\log(g_{T+k-1,j}(q^j_{T+k-1})),
\]
and
\[
B_k^{(1)}
=\sum_{j=T}^{T+k-1}
\log(g_{T+k-1,j}(q^j_{T+k-1})),
\]
with
\[
u_k=\inf\{n\in\N:m_n\ge k\}.
\]
The above inequality follows by retaining only the block
$j=T-u_k,\dots,T+k-1$ in the maximum. Since $m_n=\floor{e^{n/s}}$, we have
$u_k=s\log k+O(1)$ as $k\to\infty$. In particular, after increasing the finite
constant $k_1$ below if necessary, we may assume that $u_k\le k$ for all
$k\ge k_1$. All finitely many smaller values of $k$ will be absorbed into the
constant term.
Let
\[
\ell_k:=\log(k+1).
\]
By \Cref{assmp:upper-calib-real}, for $k\ge 1$,
\begin{align*}
B_k^{(1)}
&=\sum_{j=T}^{T+k-1}
\log(1+\lambda_{T+k-1,j}(q_{T+k-1}^j-0.5))\\
&\geq
\max_{\lambda\in[-1,1]}
\sum_{j=T}^{T+k-1}
\log(1+\lambda(q_{T+k-1}^j-0.5))
-C^\prime \log(k+1)\\
&\geq -C^\prime \ell_k,
\end{align*}
where the last inequality follows by taking $\lambda=0$. Again, by
\Cref{assmp:upper-calib-real},
\begin{align*}
B_k^{(0)}
&=\sum_{j=T-u_k}^{T-1}
\log(1+\lambda_{T+k-1,j}(q_{T+k-1}^j-0.5))\\
&\geq
\max_{\lambda\in[-1,1]}
\sum_{j=T-u_k}^{T-1}
\log(1+\lambda(q_{T+k-1}^j-0.5))
-C^\prime \log(u_k+1)\\
&\geq B_k^{(2)}-C^\prime \ell_k,
\end{align*}
where
\[
B_k^{(2)}
=
\sum_{j=T-u_k}^{T-1}
\log(1+\lambda^{**}_k(q_{T+k-1}^j-0.5)),
\]
with $\lambda^{**}_k$ as defined in \eqref{eq:lammbda-opt}. Hence, for all
$k\ge 1$,
\[
B_k^{(0)}+B_k^{(1)}
\ge B_k^{(2)}-2C^\prime \ell_k.
\]

Therefore,
\begin{align*}
&\mathbb{E}_{F_0,T,F_1}
\left(|\mathcal{C}_{up}\cap\{T,\cdots,\tau\}|
\mid T\le\tau \le T+m_T\right)\\
&\leq
\frac{
\sum_{t=T}^{T+m_T}
\mathbb{P}_{F_0,T,F_1}
(t\in \mathcal{C}_{up},\tau\ge t)
}{
\mathbb{P}_{F_0,T,F_1}(T\le\tau \le T+m_T)
}\\
&\leq
\frac{
1+
\sum_{k=1}^{m_T}
\mathbb{P}_{F_0,T,F_1}
\left(
B_k^{(0)}+B_k^{(1)}
\le -\log(\beta r_{T+k}), \tau\ge T+k
\right)
}{
\mathbb{P}_{F_0,T,F_1}(T\le\tau \le T+m_T)
}\\
&\leq
\frac{
1+
\sum_{k=1}^{m_T}
\mathbb{P}_{F_0,T,F_1}
\left(
B_k^{(2)}
\le -\log(\beta r_\tau)+2C^\prime \ell_k
\right)
}{
\mathbb{P}_{F_0,T,F_1}(T\le\tau \le T+m_T)
},
\end{align*}
because $r_t=\mathbb{P}_{F_0,\infty}(\tau\ge t)$ is non-increasing in $t$, and
on the event $\{\tau\ge T+k\}$ we have $r_{T+k}\ge r_\tau$.

Let $\mu_k=\mathbb{E}(B_k^{(2)})$. From \Cref{lem:kl}, for all sufficiently
large $k$,
\begin{equation}
\label{eq:mu_k-real-post}
\mu_k
\ge
\left(\frac{k\delta}{k+u_k}\right)^2 u_k
\ge \frac{\delta^2 u_k}{4},
\end{equation}
where the last inequality uses $u_k\le k$.

Now, using that $r_\tau$ is super-uniform under \Cref{assmp:detector}
(see \Cref{lem:rtau}), for any $c>0$,
\begin{align*}
&\mathbb{P}_{F_0,T,F_1}
\left(
B_k^{(2)}
\le -\log(\beta r_\tau)+2C^\prime \ell_k
\right)\\
&=
\mathbb{P}_{F_0,T,F_1}
\left(
B_k^{(2)}
\le -\log(\beta r_\tau)+2C^\prime \ell_k,\,
r_\tau\ge e^{-cu_k}
\right)\\
&\qquad+
\mathbb{P}_{F_0,T,F_1}(r_\tau<e^{-cu_k})\\
&\le
\mathbb{P}_{F_0,T,F_1}
\left(
B_k^{(2)}
\le cu_k-\log\beta+2C^\prime \ell_k
\right)
+
e^{-cu_k}.
\end{align*}
Since $u_k=s\log k+O(1)$, the second term satisfies
\[
e^{-cu_k}\le C k^{-cs}
\]
for a finite constant $C$ independent of $T$.
Choose
\[
c=\frac1s+\frac{\delta^2}{16}.
\]
By assumption,
\[
s>\frac{16(2C^\prime+1)}{\delta^2},
\]
and hence
\[
\frac{\delta^2s}{16}-2C^\prime-1>0.
\]
Define
\[
k_1
=
\inf\left\{
k\in\N:
\ell_k>
\frac{\log(1/\beta)+O(1)}
{\delta^2s/16-2C^\prime-1}
\right\},
\]
where the $O(1)$ term absorbs the harmless rounding constants coming from
$u_k=s\log k+O(1)$. Enlarging $k_1$ if needed, for all $k\ge k_1$ we have
\[
\mu_k-cu_k-2C^\prime \ell_k+\log\beta
>
\frac{\delta^2s}{8}\ell_k.
\]

We now split the sum into two regimes.

For $1\le k<k_1$, we trivially bound the probabilities by one:
\[
\sum_{k=1}^{k_1-1}
\mathbb{P}_{F_0,T,F_1}
\left(
B_k^{(2)}
\le cu_k-\log\beta+2C^\prime \ell_k
\right)
\le k_1.
\]

For $k\ge k_1$, applying \Cref{lem:bound-mcd} to the backward sum
$B_k^{(2)}$, which contains $u_k$ terms, with $L=2$ and $M=\log 2$, gives
\begin{align*}
&\mathbb{P}_{F_0,T,F_1}
\left(
B_k^{(2)}
\le cu_k-\log\beta+2C^\prime \ell_k
\right)\\
&\le
\exp\left(
-\frac{
2(\mu_k-cu_k+\log\beta-2C^\prime \ell_k)^2
}{
C_vu_k+4/k
}
\right),
\end{align*}
where $C_v=4+4(1+\log2)^2$. Therefore, for all $k\ge k_1$,
\begin{align*}
    &\mathbb{P}_{F_0,T,F_1}
\left(
B_k^{(2)}
\le cu_k-\log\beta+2C^\prime \ell_k
\right)\\
&
\le
\exp\left(
-\frac{
(\delta^2s\ell_k)^2
}{
32(C_vs\ell_k+4/k)
}
\right),
\end{align*}
where we again used $u_k=s\log k+O(1)$ and absorbed constants by increasing
$k_1$ if necessary.

Combining the bounds, we obtain
\begin{align*}
&\mathbb{E}_{F_0,T,F_1}
\left(|\mathcal{C}_{up}\cap\{T,\dots,\tau\}|
\mid T\le\tau \le T+m_T\right)\\
&\le
\frac{
1+k_1
+
\sum_{k=k_1}^{m_T}
\exp\left(
-\frac{
(\delta^2s\ell_k)^2
}{
32(C_vs\ell_k+4/k)
}
\right)
+
C\sum_{k=k_1}^{m_T} k^{-cs}
}{
\mathbb{P}_{F_0,T,F_1}(T\le\tau \le T+m_T)
}.
\end{align*}
Since
\[
cs=1+\frac{s\delta^2}{16}>1,
\]
we have
\[
\sum_{k=k_1}^{m_T}k^{-cs}=O(1).
\]
Moreover, since $\ell_k=\log(k+1)$ grows monotonically and $4/k\to0$, there
exists a finite $k^\prime$ such that $4/k\le s\ell_k$ for all $k\ge k^\prime$.
Hence,
\begin{align*}
&\sum_{k=k_1}^{m_T}
\exp\left(
-\frac{
(\delta^2s\ell_k)^2
}{
32(C_vs\ell_k+4/k)
}
\right)\\
&\le
O(1)
+
\sum_{k=1}^{m_T}
\exp\left(
-\frac{s\delta^4\ell_k}{32(C_v+1)}
\right)\\
&=
O(1)
+
\sum_{k=1}^{m_T}
(k+1)^{-\frac{s\delta^4}{32(C_v+1)}}\\
&=O(1),
\end{align*}
because
\[
s>\frac{32(C_v+1)}{\delta^4}.
\]
Finally, the denominator is bounded away from zero by assumption, and therefore
\[
\mathbb{E}_{F_0,T,F_1}
\left(|\mathcal{C}_{up}\cap\{T,\dots,\tau\}|
\mid T\le\tau \le T+m_T\right)
=O(1).
\]
This proves the desired result.
\end{proof}

\subsection{Auxiliary Lemmas}
For any $k\in[T-1],l\in[m_k-1]$, define
\begin{align}
\label{eq:lammbda-opt}
 \nonumber  & \lambda^{**}_k=\\
    &\begin{cases}
        \frac{2k\delta}{k+m_k}, ~~ \mathbb{P}(S(Z) > S(Y)) + \frac{1}{2}\mathbb{P}(S(Z) = S(Y)) = \frac{1}{2}+\delta\\
        -\frac{2k\delta}{k+m_k},   \mathbb{P}(S(Z) > S(Y)) + \frac{1}{2}\mathbb{P}(S(Z) = S(Y)) = \frac{1}{2}-\delta.
    \end{cases}
\end{align}
\begin{lemma}
\label{lem:kl}
    Assume that there exists a non-conformity score function $S(x)$ such that for independent draws $Y\sim F_0$ and $Z\sim F_1$, there exists a constant $\delta \in(0,1/2)$ such that: $|\mathbb{P}(S(Z) > S(Y)) + \frac{1}{2}\mathbb{P}(S(Z) = S(Y)) - \frac{1}{2}|= \delta.$ 
Then,
    $$\mathbb{E}[\log (1+\lambda^{**}_k(p^{T+l}_{T-k}-0.5))] \ge \left( \frac{k\delta}{k+m_k}\right)^2.$$
    consequently, for the optimal choice, $f_{T-k,T+l}(p)=1 + \lambda_{T-k,T+l}^*(p - 0.5)$ where $\lambda_{T-k,T+l}^* \in \arg\max_{\lambda \in [-1, 1]} \mathbb{E}\left[\log(1 + \lambda(p^{T+l}_{T-k} - 0.5))\right]$. Then,
    $$\mathbb{E}[\log f_{T-k,T+l}(p^{T+l}_{T-k})] \ge \left( \frac{k\delta}{k+m_k}\right)^2.$$
\end{lemma}
\begin{proof}
    We take the expectation of $p^j_t$ over the joint distribution of the data sequence and the randomizer $U_{t,j}$. By linearity of expectation:
  \begin{align*}
      \mathbb{E}[p^{T+l}_{T-k}] = \frac{1}{k+l+1} \sum_{i=T-k}^{T+l} &\Bigg( \mathbb{P}(S(X_{T+l}) > S(X_i))\\
      &+ \frac{1}{2}\mathbb{P}(S(X_{T+l}) = S(X_i)) \Bigg).
  \end{align*}

We consider two cases separately.

\textbf{Case 1: $\mathbb{P}(S(Z) > S(Y)) + \frac{1}{2}\mathbb{P}(S(Z) = S(Y)) = \frac{1}{2} + \delta.$} 

We partition the sum into three mutually exclusive groups of indices $i$:
    
    The test point itself ($i = {T+l}$): $\mathbb{P}(S(X_{T+l}) > S(X_{T+l})) + \frac{1}{2}\mathbb{P}(S(X_{T+l}) = S(X_{T+l})) = \frac{1}{2}.$
    
    Other post-change points ($T \le i < {T+l}$): Because $X_i, X_{T+l} \stackrel{i.i.d.}{\sim} F_1$, they are strictly exchangeable. By symmetry, $\mathbb{P}(S_{T+l} > S_i) = \mathbb{P}(S_i > S_{T+l})$. Since $\mathbb{P}(S_{T+l} > S_i) + \mathbb{P}(S_i > S_{T+l}) + \mathbb{P}(S_{T+l} = S_i) = 1$, it follows algebraically that $\mathbb{P}(S_{T+l} > S_i) + \frac{1}{2}\mathbb{P}(S_{T+l} = S_i) = \frac{1}{2}$. There are $l$ such points.
    
    Pre-change points ($T-k \le i < T$): Because $X_i \sim F_0$ and $X_{T+l} \sim F_1$, we have $\mathbb{P}(S_{T+l} > S_i) + \frac{1}{2}\mathbb{P}(S_{T+l} = S_i)=\frac{1}{2} + \delta$ by assumption. There are $k$ such points.
    
    Summing the expectations across all $k+l+1$ points:$$\mathbb{E}[p^{T+l}_{T-k}] = \frac{1}{k+l+1} \left[ \frac{1}{2} + \frac{l}{2} + k\left(\frac{1}{2} + \delta\right) \right] = \frac{1}{2} + \frac{k}{k+l+1}\delta.$$
    Let $Y = p_{T-k}^{T+l} - 0.5$. Because the conformal p-value satisfies $p_{T-k}^{T+l} \in [0, 1]$ almost surely, the shifted variable is deterministically bounded:$$Y \in [-0.5, 0.5].$$
    Since $l\leq m_k-1$,
    $$\mathbb{E}[Y] = \mathbb{E}[p_{T-k}^{T+l}] - 0.5 = \frac{k}{k+l+1}\delta\geq \frac{k\delta}{k+m_k} := \epsilon.$$
    Because $\delta \le 0.5$, we get $\epsilon \le 0.5$.
We have $\lambda^{**}_k = 2\epsilon$.
Because $0 < \epsilon \le 0.5$, our chosen $\lambda^{**}_k$ satisfies $\lambda^{**}_k \in (0, 1]$. 
Because $Y \in [-0.5, 0.5]$ and $\lambda^{**}_k\in (0, 1]$, their product strictly satisfies almost surely:$$ \lambda^{**}_kY \in [-0.5, 0.5].$$
In this restricted domain $[-0.5, 0.5]$, we use the inequality:$$\log(1+x) \ge x - x^2 \quad \text{for all } x \ge -0.5.$$
Therefore,
$$\mathbb{E}[\log(1 + \lambda^{**}_kY)] \ge \mathbb{E}[\lambda^{**}_kY - (\lambda^{**}_k)^2 Y^2] \geq \lambda^{**}_k\epsilon - (\lambda^{**}_k)^2 \mathbb{E}[Y^2].$$
Because $Y \in [-0.5, 0.5]$, its second moment is deterministically bounded by its maximum possible squared value: $\mathbb{E}[Y^2] \le (0.5)^2 = 0.25$. Since $(\lambda^{**}_k)^2 > 0$, we have
$$\mathbb{E}[\log(1 + \lambda^{**}_kY)] \ge \lambda^{**}_k\epsilon - 0.25 (\lambda^{**}_k)^2=2\epsilon^2- 0.25\times(2\epsilon)^2=\epsilon^2.$$
Because $\lambda_{T-k,T+l}^*$ is the maximizer over $[-1,1]$, it must perform at least as well as $\lambda^{**}_k= 2\epsilon\in[-1,1]$, and hence:
\begin{align*}
   & \mathbb{E}[\log f_{T-k,T+l}(p^{T+l}_{T-k})] =\mathbb{E}\left[\log(1 + \lambda^*_{T-k,T+l} Y)\right] \\
    &\ge \mathbb{E}[\log(1 + \lambda^{**}_kY)]\ge \epsilon^2 = \left( \frac{k\delta}{k+m_k} \right)^2.
\end{align*}

\textbf{Case 2: $\mathbb{P}(S(Z) > S(Y)) + \frac{1}{2}\mathbb{P}(S(Z) = S(Y)) = \frac{1}{2} - \delta.$} 

We similarly obtain
$$\mathbb{E}[p^{T+l}_{T-k}] = \frac{1}{k+l+1} \left[ \frac{1}{2} +\frac{l}{2}+ k\left(\frac{1}{2} - \delta\right) \right] = \frac{1}{2} - \frac{k\delta}{k+l+1}.$$

The exact expectation of $Y$ is:$$\mathbb{E}[Y] = \mathbb{E}[p_{T-k}^{T+l}] - 0.5 = -\frac{k}{k+l+1}\delta\leq -\frac{k}{k+m_k}\delta =- \epsilon.$$
    Because $\delta \le 0.5$, we get $\epsilon \le 0.5$.
We have $\lambda^{**}_k = -2\epsilon$.
Because $0 < \epsilon \le 0.5$, our chosen $\lambda^{**}_k$ satisfies $\lambda^{**}_k \in [-1, 0]$. 
Because $Y \in [-0.5, 0.5]$ and $\lambda^{**}_k \in [-1, 0)$, their product strictly satisfies almost surely:$$ \lambda^{**}_k Y \in [-0.5, 0.5].$$
In this restricted domain $[-0.5, 0.5]$, we use the inequality:$$\log(1+x) \ge x - x^2 \quad \text{for all } x \ge -0.5.$$
Since $\lambda^{**}_k<0$,
$\mathbb{E}[\log(1 + \lambda^{**}_k Y)] \ge \mathbb{E}[\lambda^{**}_k Y - (\lambda^{**}_k)^2 Y^2] \geq -\lambda^{**}_k \epsilon - (\lambda^{**}_k)^2 \mathbb{E}[Y^2].$
Because $Y \in [-0.5, 0.5]$,  $\mathbb{E}[Y^2] \le (0.5)^2 = 0.25$.

Since $(\lambda^{**}_k)^2 > 0$, we have
$\mathbb{E}[\log(1 + \lambda^{**}_k Y)] \ge -\lambda^{**}_k\epsilon - 0.25 (\lambda^{**}_k)^2=2\epsilon^2- 0.25\times(-2\epsilon)^2=\epsilon^2.$

Because $\lambda_{T-k,T+l}^*$ is the maximizer over $[-1,1]$, it must perform at least as well as $\lambda^{**}_k= -2\epsilon\in[-1,1]$, and hence:
\begin{align*}
    &\mathbb{E}[\log f_{T-k,T+l}(p^{T+l}_{T-k})] =\mathbb{E}\left[\log(1 + \lambda^*_{T-k,T+l} Y)\right] \\
    &\ge \mathbb{E}[\log(1 + \lambda_k^{**} Y)]\ge \epsilon^2 = \left( \frac{k\delta}{k+m_k} \right)^2.
\end{align*}
\end{proof}

\begin{lemma}
\label{lem:bound-mcd}
Assume that there exists a non-conformity score function $S(x)$ such that scores are all distinct almost surely. Additionally, assume that, for each $t$ and $j\geq t$, the calibrator $f_{t,j}$ is such that $\log f_{t,j}$ is universally bounded by $M$ and is a $L$-Lipschitz function. Then, under \Cref{assmp:indep}, for any $u>0$, $k\in[T-1]$ and $m_k\leq k$,  $S_k=\sum_{i=T}^{T+m_k-1}\log f_{T-k,i}\!\left(p^{\,i}_{T-k}\right)$ satisfies
\[
P\!\left(
S_k-\mathbb E S_k \le -u
\right)
\le
\exp\!\left(
-\frac{2u^2}{C_vm_k+L^2/k}
\right),
\]
where $C_v=L^2 + (2M+L)^2$.
\end{lemma}
\begin{proof}

Fix $t=T-k$ and define the random vector
\[
Z :=
(X_{T-k},\dots,X_{T+m_k-1},
U_{t,T},\dots,U_{t,T+m_k-1}),
\]
where the auxiliary variables
$\{U_{t,i}\}$ are i.i.d.\ $\mathrm{Unif}(0,1)$
and independent of all observations.
Under $P_{F_0,T,F_1}$, all coordinates of $Z$
are mutually independent.
Since the scores are almost surely unique,
each conformal $p$-value admits the representation
\[
p_t^i
=
\frac{1}{i-t+1}
\left[
\sum_{r=t}^{i-1}\mathbf 1\{S(X_i)>S(X_r)\}
+
U_{t,i}
\right],
\]
hence there exists a deterministic measurable function $G$
such that
\[
S_k = G(Z).
\]
Therefore, McDiarmid's inequality applies once we bound the
effect of modifying one coordinate of $Z$.
Define
\[
c_n
:=
\sup_{z,z'}
|G(z)-G(z')|,
\]
where $z,z'$ differ only in coordinate $n$.

We bound sensitivities separately for observations and
auxiliary uniforms.

\textbf{Case 1: perturbing a pre-change observation.}

Let $T-k \le j < T$.
Only rank comparisons involving $X_j$ are affected. Let $\tilde p_t^i$ denote the new p-value obtained when $X_j$ is replaced by $\tilde X_j$.
For each $i\ge T$,
\[
|p_t^i-\tilde p_t^i|
\le \frac{1}{i-t+1}.
\]
By Lipschitz continuity,
\[
|\log f_{t,j}(p_t^i)-\log f_{t,j}(\tilde p_t^i)|
\le
\frac{L}{i-t+1}.
\]
 Note that $X_j$ corresponds to $n=j-T+k+1$-th coordinate of $Z$. Hence
\[
c_{j-T+k+1}
\le
L\sum_{i=T}^{T+m_k-1}\frac{1}{i-t+1}
=
L\sum_{r=k+1}^{k+m_k}\frac1r
\le
\frac{m_kL}{k}.
\]
Thus
\[
\sum_{n=1}^{k}c_n^2
\le
\frac{m_k^2 L^2}{k^2}\times k=\frac{m_k^2 L^2}{k}.
\]

\textbf{Case 2: perturbing a post-change observation.}
Let $T\le j<T+m_k$ and set $r=j-t+1$.

\textbf{Direct effect.}
The term $\log f_{t,j}(p_t^i)$ can get changed arbitrarily. So, we have
$|\log f_{t,j}(p_t^i)-\log f_{t,j}(\tilde p_t^i)|\leq 2M$.

\textbf{Propagation effect.}
For every $i>j$,
\[
|p_t^i-\tilde p_t^i|
\le \frac{1}{i-t+1},
\]
so
\[
|\log f_{t,j}(p_t^i)-\log f_{t,j}(\tilde p_t^i)|
\le
\frac{L}{i-t+1}.
\]
 Note that $X_j$ corresponds to $n=j-T+k+1$-th coordinate of $Z$. Hence
\[
c_{j-T+k+1}
\le
2M
+
L\sum_{i=j+1}^{T+m_k-1}\frac{1}{i-t+1}.
\]
Because $i > j\geq T$, each term is strictly less than $1/k$. There are at most $m_k - 1$ terms in this propagation sum. Hence:$$c_{j-T+k+1} \le 2M + \frac{m_k L}{k}.$$
Summing the squared deviations over the $m_k$ post-change candidates:$$\sum_{n=k+1}^{k+m_k} c_n^2 \le m_k \left( 2M + \frac{m_k L}{k} \right)^2.$$

\textbf{Case 3: perturbing an auxiliary uniform.}
For fixed $i\in\{T,\cdots,T+m_k-1\}$,
changing  $U_{t,i}$ affects only the single
$p$-value $p_t^i$. Therefore,
\[
|p_t^i-\tilde p_t^i|
\le \frac{1}{i-T+k+1},
\]
and hence for $n=k+m_k+i-T+1$,
\[
c_{k+m_k+i-T+1} \le \frac{L}{i-T+k+1}.
\]
Summing,
\[
\sum_{n=k+m_k+1}^{k+2m_k} c_n^2
\le
L^2\sum_{r=k+1}^{k+m_k}\frac1{r^2}
\le
\frac{L^2m_k}{k^2}.
\]
\textbf{Total variance proxy.}
Combining all three contributions, the total variance proxy $V_k := \sum_{n=1}^{k+2m_k} c_n^2$ is exactly bounded by:$$V_k
\le
\frac{m_k^2 L^2}{k}
+
m_k
\left(
2M+\frac{m_k L}{k}
\right)^2
+
\frac{L^2m_k}{k^2}.$$
Since $m_k \le k$, therefore, 
$$V_k
\le
m_k \left[ L^2 + (2M+L)^2 \right] + \frac{L^2}{k}=C_vm_k+L^2/k.$$
\textbf{Application of McDiarmid's inequality.}
Since $S_k=G(Z)$ is a function of independent coordinates
with bounded differences,
McDiarmid's inequality gives
\[
P\!\left(
S_k-\mathbb E S_k \le -u
\right)
\le
\exp\!\left(
-\frac{2u^2}{C_vm_k+L^2/k}
\right).
\]
Note: In McDiarmid, if we only use boundedness, the variance bound would be $O(k)$. Using both boundedness and Lipschitz continuity, we derived a variance bound of $O(m_k)$. This is much better than the naive $O(k)$ bound, especially when $m_k<<k$.
\end{proof}

\begin{lemma}
\label{lem:rtau}
Under \Cref{assmp:detector}, the random variable $r_\tau$ satisfies
\[
\mathbb{P}_{F_0,T,F_1}(r_\tau \le x) \le x,\qquad \forall x \in [0,1].
\]
Equivalently, $r_\tau$ stochastically dominates a $\mathrm{Unif}(0,1)$ random variable.
\end{lemma}

\begin{proof}
Since
\[
r_t=\mathbb P_{F_0,\infty}(\tau\ge t),
\]
the sequence $(r_t)_{t\ge1}$ is nonincreasing. 
Now fix $x\in[0,1]$, and define
\[
m:=\min\{t\in\mathbb N:\ r_t\le x\},
\]
with the convention $m=\infty$ if the set is empty. Since $(r_t)$ is nonincreasing,
\[
\{r_\tau\le x\}=\{\tau\ge m\}.
\]
Therefore,
\[
\mathbb P_{F_0,T,F_1}(r_\tau\le x)=\mathbb P_{F_0,T,F_1}(\tau\ge m)\leq \mathbb P_{F_0,\infty}(\tau\ge m)=r_m,
\]
where the inequality follows from \Cref{assmp:detector}. By definition of $m$, we have $r_m\le x$. Hence,
\[
\mathbb P_{F_0,T,F_1}(r_\tau\le x)\le x.
\]
This completes the proof.
\end{proof}

\section{Additional theoretical result}
\label{sec:app-theory}
\begin{theorem}
\label{thm:score-tv}
Let $Y \sim P_0$ and $Z \sim P_1$ be independent random variables on a measurable space $(\mathcal X,\mathcal F)$. Then
$\frac{1}{2}\operatorname{TV}(P_1,P_0)\le\sup_{S}\left|P(S(Z)>S(Y))+\frac12 P(S(Z)=S(Y))-\frac12\right|
\le\operatorname{TV}(P_1,P_0),$
where the supremum is over all measurable functions
$S:\mathcal X \to \mathbb R$. Moreover,  every strictly increasing transform of
\[
L(x):=\frac{dP_1}{dP_0+dP_1}(x)
\]
attains the supremum.
\end{theorem}

\begin{proof}
For a measurable score function $S:\mathcal X \to \mathbb R$, define
\[
\Gamma(S):=P(S(Z)>S(Y))+\frac12 P(S(Z)=S(Y))-\frac12.
\]
Since
\[
P(S(Z)>S(Y))+P(S(Z)<S(Y))+P(S(Z)=S(Y))=1,
\]
we obtain
\begin{align*}
\Gamma(S)&=P(S(Z)>S(Y))+\frac12P(S(Z)=S(Y))-\frac12\\
&=\frac12\Bigl(P(S(Z)>S(Y))-P(S(Z)<S(Y))\Bigr).
\end{align*}
Hence
\[
|\Gamma(S)|=\frac12\left|P(S(Z)>S(Y))-P(S(Z)<S(Y))\right|.
\]
\textbf{Upper bound. } We first prove the upper bound. 
Define
\[
h(y,z):=\operatorname{sgn}(S(z) - S(y))=\mathds{1}(S(z)>S(y))-\mathds{1}(S(z)<S(y)).
\]
Then $|h(u)|\le 1$ for every $u \in \mathbb R$. Moreover,
\begin{align}
\label{eq:gammas}
  \nonumber  \Gamma(S)&=\frac12 (\mathbb E_{(Y,Z)\sim P_0\times P_1}[h(Y,Z)]\\
    &=\frac12 (\mathbb E_{(Y,Z)\sim P_0\times P_1}[h(Y,Z)]-\mathbb E_{(Y,Z)\sim P_0\times P_0}[h(Y,Z)]),
\end{align}
because $\mathbb E_{(Y,Z)\sim P_0\times P_0}[h(Y,Z)]=0$.
Using the variational characterization of total variation distance,
\begin{align*}
    &\mathbb E_{(Y,Z)\sim P_0\times P_1}[h(Y,Z)]-\mathbb E_{(Y,Z)\sim P_0\times P_0}[h(Y,Z)]\\
    &
\le
2\|h\|_\infty \operatorname{TV}(P_0\times P_1, P_0\times P_0).
\end{align*}
Since $\|h\|_\infty \le 1$ and $\operatorname{TV}(P_0\times P_1, P_0\times P_0)=\operatorname{TV}( P_1,  P_0)$, it follows that
\[
|\Gamma(S)|
\le
\operatorname{TV}(P_1,P_0),
\]
i.e., $\sup_{S}\left|P(S(Z)>S(Y))+\frac12 P(S(Z)=S(Y))-\frac12\right|
\le\operatorname{TV}(P_1,P_0).$

\textbf{Lower bound. } We now prove the lower bound.
Recall that
\[
\operatorname{TV}(P_1,P_0)=\sup_{A\in\mathcal F}|P_1(A)-P_0(A)|.
\]
Fix a measurable set $A \in \mathcal F$, and define
\[
S(x):=\mathbf 1_A(x).
\]
Let $a:=P_1(A),$ and $b:=P_0(A).$ Since $S$ takes values in $\{0,1\}$,
\[
P(S(Z)>S(Y))=P(S(Z)=1,S(Y)=0)=a(1-b),
\]
and
\begin{align*}
&P(S(Z)=S(Y))\\
&=P(S(Z)=1,S(Y)=1) +P(S(Z)=0,S(Y)=0) \\
&=ab+(1-a)(1-b).
\end{align*}
Therefore,
\begin{align*}
\Gamma(S)
&=a(1-b)+\frac12\bigl(ab+(1-a)(1-b)\bigr)-\frac12 =\frac12(a-b).
\end{align*}
Consequently,
\[
|\Gamma(S)|=
\frac12 |P_1(A)-P_0(A)|.
\]

Taking the supremum over measurable sets $A$ gives
\(
\sup_S |\Gamma(S)|\ge\frac12 \operatorname{TV}(P_1,P_0).
\)
Combining the upper and lower bounds completes the proof.

\textbf{Scores attaining the supremum.}
Let $\mu = P_0 + P_1$. By definition, $P_0 \ll \mu$ and $P_1 \ll \mu$. By the Radon-Nikodym theorem, the derivatives (densities) exist everywhere on $\mathcal{X}$:$$f_1(x) = \frac{dP_1}{d\mu}(x),~f_0(x) = \frac{dP_0}{d\mu}(x).$$
From \eqref{eq:gammas}
$$\Gamma(S) = \frac{1}{2} \int_{\mathcal{X}} \int_{\mathcal{X}} \operatorname{sgn}(S(z) - S(y)) f_1(z) f_0(y) d\mu(z) d\mu(y).$$
Because $z$ and $y$ are simply dummy variables of integration, we can swap them without changing the value of the integral:$$\Gamma(S) = \frac{1}{2} \int_{\mathcal{X}} \int_{\mathcal{X}} \operatorname{sgn}(S(y) - S(z)) f_1(y) f_0(z) d\mu(y) d\mu(z).$$
Since the signum function is strictly antisymmetric, $\operatorname{sgn}(S(y) - S(z)) = -\operatorname{sgn}(S(z) - S(y))$. Substituting this gives us a second expression for $\Gamma(S)$:$$\Gamma(S) = -\frac{1}{2} \int_{\mathcal{X}} \int_{\mathcal{X}} \operatorname{sgn}(S(z) - S(y)) f_0(z) f_1(y) d\mu(z) d\mu(y).$$
Now, adding the first integral and the second integral together (which equals $2\Gamma(S)$), and dividing by $2$, we get $\Gamma(S) =$
$$ \frac{1}{4} \int_{\mathcal{X}} \int_{\mathcal{X}} \operatorname{sgn}(S(z) - S(y)) \Big[f_1(z) f_0(y) - f_0(z) f_1(y)\Big] d\mu(z) d\mu(y).$$
Note that by definition, $f_0(x) + f_1(x) = 1 \quad \mu\text{-a.e.}$. Therefore,
\begin{align*}
    &f_1(z) f_0(y) - f_0(z) f_1(y)\\
    &= f_1(z) \big(1 - f_1(y)\big) - \big(1 - f_1(z)\big) f_1(y)= f_1(z) - f_1(y).
\end{align*}
Substitute this back into our integral for $\Gamma(S)$:$$\Gamma(S) = \frac{1}{4} \int_{\mathcal{X}} \int_{\mathcal{X}} \operatorname{sgn}(S(z) - S(y)) \Big[f_1(z) - f_1(y)\Big] d\mu(z) d\mu(y).$$
To achieve the supremum over all measurable functions $S$, the integrand must be maximized almost everywhere. This occurs if and only if the sign of our scoring function's difference perfectly matches the sign of the density difference:$$\operatorname{sgn}(S(z) - S(y)) = \operatorname{sgn}(f_1(z) - f_1(y)) \quad \mu\text{-a.e.}$$By definition, this equality holds if we simply set our scoring function to be the density itself:$$S^*(x) = f_1(x) = L(x).$$
Furthermore, because the $\operatorname{sgn}$ function only cares about strict ordinal ranking, any strictly increasing transformation $T: \mathbb{R} \to \mathbb{R}$ applied to $L(x)$ will preserve the exact same inequality:$$L(z) > L(y) \iff T(L(z)) > T(L(y)).$$
Therefore, every strictly increasing transformation of $L(x) = \frac{dP_1}{dP_0 + dP_1}(x)$ achieves the supremum.
\end{proof}

\section{A detection-algorithm-agnostic approach to estimate $\P_{F_0,\infty}[\tau\geq t]$}
\label{sec:app-rt-traing-data}

In \Cref{sec:rt}, we assumed that the change detector is distribution-free. However, one may wish to use a change detection procedure that is not distribution-free while still seeking distribution-free post-detection guarantees.
One possible approach is to assume access to a historical sample
\[Y_1,\cdots,Y_m \sim F_0\]
where $m$ can be any integer, but it should ideally be large (to have power). 
We assume that the historical data is independent of the sample $X_1, X_2,\cdots, X_\tau$. This is a common assumption in the sequential changepoint detection literature.
For each permutation $\pi$ of the historical data, run the change detector and let $\tau^\pi$ be the stopping time. Define
\begin{equation}
    q_t^\pi=\begin{cases}
        \mathds{1}(\tau^\pi \geq t), &\text{ if } t\leq m\\
        0,&\text{ if } t> m.
    \end{cases}
\end{equation}
Then define
\begin{equation}
\label{rt}
    \hat{r}_t=\frac{1}{m!}\sum_{\pi}q_t^\pi.
\end{equation}
\begin{proposition}
 The estimator $\hat{r}_t$ defined in \eqref{rt} is an unbiased or negatively biased estimator of $\P_{F_0,\infty}[\tau\geq t]$.
\end{proposition}
\begin{proof}
Let $Y = (Y_1, \dots, Y_m)$ be the historical data where $Y_i \overset{i.i.d.}{\sim} F_0$. By the linearity of expectation and the exchangeability of $Y$:
\begin{align*}
\mathbb{E}[\hat{r}_t] &= \frac{1}{m!} \sum_{\pi} \mathbb{E}[\mathds{1}(\tau^\pi \geq t) \mathds{1}(t \leq m)] 
= \mathbb{P}_{F_0, \infty}(\tau \geq t, t \leq m).
\end{align*}
We consider two cases for the value of $t$:
\begin{enumerate}
    \item If $t \leq m$, then $\mathbb{E}[\hat{r}_t] = \mathbb{P}_{F_0, \infty}(\tau \geq t)$, which is an unbiased estimate.
    \item If $t > m$, then $\mathbb{E}[\hat{r}_t] = 0$. Since $\mathbb{P}_{F_0, \infty}(\tau \geq t) \geq 0$, the estimator is negatively biased (or unbiased if the probability is exactly zero).
\end{enumerate}
In all cases, $\mathbb{E}[\hat{r}_t] \leq \mathbb{P}_{F_0, \infty}(\tau \geq t)$.
\end{proof}

Another asymptotic alternative is based on a bootstrap approximation.
Draw $B$ many bootstrap sequences $\{Y_n^{*j}\}_n$, for $j=1,..,B$ and run the change detector to compute $\mathds{1}(\tau^j \geq t)$, for $t=1,\cdots,\tau$ and $j=1,\cdots,B$. Then define
\begin{equation}
\label{rt-boot}
    \hat{r}_t=\frac{1}{B}\sum_{j=1}^B\mathds{1}(\tau^j \geq t).
\end{equation}
Under standard consistency conditions for the bootstrap approximation of the null distribution, we have an asymptotic guarantee $\mathbb{E}[\hat{r}_t]\to\mathbb{P}_{F_0, \infty}(\tau \geq t),$ as $m,B\to\infty$.

\section{Additional experimental results}
\label{sec:app-expt}
\subsection{Varying the changepoint and the detection threshold}
We consider the setup with the Gaussian mean change (N(-1,1) to N(1,1)) setting from \Cref{sec:expt}, but now in a triangular-array framework where both the changepoint $T$ and the detection thresholds $A$ of the conformal-CUSUM detector grow proportionally. The conformal-CUSUM detector we use here is the following:
\begin{equation}
\label{eq:conformal-cusum}
    \tau_A=\inf\{t\in\N: \max_{i=0,\cdots,n-1}S_n/S_i\geq A\},
\end{equation}
where $S_t=\prod_{j=1}^tf_{1,j}(p_1^j)$ is the conformal test martingale for the sequence $X_1,\cdots,X_t$. Note that a larger $A$ leads to a longer detection delay. The average (of $100$ independent runs) conditional size and coverage of the two-sided confidence sets with target coverage $90\%$ (i.e., \eqref{ci-bothside-2alpha} with $\alpha=0.05$) is reported in \Cref{tab:varying-ratio-1,tab:varying-ratio-2}. In \Cref{tab:varying-ratio-1}, we consider relatively small values of $A/T$ ratio and $s$. In this regime, the confidence set size increases with $T$, although its size relative to $T$ shrinks toward zero. In contrast, for \Cref{tab:varying-ratio-2}, we choose a larger $A/T$ ratio and $s$, and see that the conditional size of the sets stays bounded as $T$ increases.

 \begin{table}[!ht]
    \centering
    \caption{Average results with smaller $A/T$ ratio and $s=4$.
    }
\label{tab:varying-ratio-1}
    \resizebox{\linewidth}{!}{
    \begin{tabular}{cccccc}
    \toprule
    \addlinespace
$T$ & $A$ & \specialcell{Coverage  (cond. \\on $\tau\geq T$)} & \specialcell{Size (cond. on\\ $\tau\geq T+s\log T$)} &\specialcell{Relative size (cond. \\on $\tau\geq T+s\log T$)} & \specialcell{Delay (cond. \\on $\tau\geq T$)}  \\
\midrule
500 & 10000 & 0.93 & 21.15 & 0.0423 & 35.75\\
1500 & 30000 & 0.98 & 60.32 & 0.0402 & 39.74\\
2500 & 50000 & 1 & 98.76 & 0.0395 & 42.25\\
3500 & 70000 & 1 &  119.41 & 0.0341 & 44.93 \\
4500 & 90000 & 1 & 130.97 & 0.0291 & 45.82 \\
\bottomrule
 \end{tabular}}
 \end{table}

 \begin{table}[!ht]
    \centering
    \caption{Average results with larger $A/T$ ratio and $s=6$.
    }
\label{tab:varying-ratio-2}
    \resizebox{\linewidth}{!}{
    \begin{tabular}{cccccc}
    \toprule
    \addlinespace
$T$ & $A$ & \specialcell{Coverage  (cond. \\on $\tau\geq T$)} &   \specialcell{Size (cond. on\\ $\tau\geq T+s\log T$)} &\specialcell{Relative size (cond. \\on $\tau\geq T+s\log T$)} & \specialcell{Delay (cond. \\on $\tau\geq T$)}  \\
\midrule
500 & 50000 & 0.92 & 8.56 & 0.0171 & 42.58\\
1500 & 150000 & 0.93& 8.85 & 0.0059 & 47.75\\
2500 & 250000 & 0.93 & 8.97 & 0.0035 & 50.02\\
3500 & 350000 & 0.91 & 8.81 & 0.0025  & 52.36\\
4500 & 450000 & 0.94 & 8.92 & 0.0019 & 58.17\\
\bottomrule
 \end{tabular}}
 \end{table}

\subsection{Varying the signal strength}
In the same setup as the previous experiment, we next vary the true pre- and post-change means while keeping $T$ and $A$ fixed at two different levels: $T=500, A=50000$ and $T=1000, A=100000$. Average results of $100$ independent runs are shown in \Cref{tab:varying-signal-1,tab:varying-signal-2}. As predicted by \Cref{thm:twoside-asymp-real}, the conditional size of the confidence sets increases as the signal strength decreases (i.e., as $\delta$ in the theorem decreases) in both tables. Nevertheless, even for the weaker signals considered here, the size of the confidence sets remains substantially small, indicating that the proposed method continues to provide informative localization guarantees.
 \begin{table}[!ht]
    \centering
    \caption{Average results with $T=500, A=50000$.
    }
\label{tab:varying-signal-1}
    \resizebox{\linewidth}{!}{
    \begin{tabular}{cccccc}
    \toprule
    \addlinespace
Pre-change & Post-change & \specialcell{Coverage  (cond. \\on $\tau\geq T$)}  & \specialcell{Size (cond. on\\ $\tau\geq T+5\log T$)} &\specialcell{Relative size (cond. \\on $\tau\geq T+5\log T$)} & \specialcell{Delay (cond. \\on $\tau\geq T$)}  \\
\midrule
N(-1,1) & N(1,1) & 0.93 & 10.73 & 0.0215 & 40.20 \\
N(-0.75,1) & N(0.75,1) & 0.94 & 15.13 & 0.0303 & 47.86 \\
 N(-0.5,1) & N(0.5,1) & 0.96 & 26.11 & 0.0522 & 67.19\\
\bottomrule
 \end{tabular}}
 \end{table}

 \begin{table}[!ht]
    \centering
    \caption{Average results with $T=1000, A=100000$.
    }
\label{tab:varying-signal-2}
    \resizebox{0.95\linewidth}{!}{
    \begin{tabular}{cccccc}
    \toprule
    \addlinespace
Pre-change & Post-change & \specialcell{Coverage  (cond. \\on $\tau\geq T$)} & \specialcell{Size (cond. on\\ $\tau\geq T+5\log T$)} &\specialcell{Relative size (cond. \\on $\tau\geq T+5\log T$)} & \specialcell{Delay (cond. \\on $\tau\geq T$)}  \\
\midrule
N(-1,1) & N(1,1) & 0.92 & 10.84    & 0.01084 &  44.44 \\
N(-0.75,1) & N(0.75,1) & 0.94 & 14.60   & 0.01460 &   52.00 \\
 N(-0.5,1) & N(0.5,1) & 0.97 & 27.26   & 0.02726 & 71.68\\
\bottomrule
 \end{tabular}}
 \end{table}